%% file: neurips_2026.tex
\documentclass{article}

\PassOptionsToPackage{numbers, compress}{natbib}
\usepackage[main, final]{neurips_2026}
\usepackage{microtype}
\usepackage{graphicx}
\usepackage{subcaption}
\usepackage{booktabs} % for professional tables

\input{math_commands.tex}

\usepackage{amsmath}
\usepackage{amssymb}

\usepackage{mathtools}
\usepackage{algorithm}
\usepackage{algorithmic}
\usepackage{amsthm}

\usepackage{fontawesome5}
\usepackage{colortbl}
\usepackage{xcolor}
\arrayrulecolor{black}
\usepackage{url}
\usepackage{booktabs}
\usepackage{enumitem}
\usepackage{amsmath}
\usepackage{amsthm}
\usepackage{amssymb}
\usepackage{natbib}
\usepackage{xspace}
\usepackage{adjustbox}
\usepackage{multirow}
\usepackage{wrapfig}
\usepackage{pifont}
\usepackage{sidecap}
\usepackage{graphicx}
\usepackage{colortbl} % To add color to v-lines
\usepackage{caption}        % 支持 captionof 用于 wrapfigure 
\usepackage{soul}
\usepackage{makecell}  % 放在导言区

\theoremstyle{plain}
\newtheorem{theorem}{Theorem}[section]
\newtheorem{proposition}[theorem]{Proposition}

\theoremstyle{definition}

\theoremstyle{remark}

\usepackage[textsize=tiny]{todonotes}
\definecolor{midblue}{rgb}{0.21,0.49,0.74}
\definecolor{alicegreen}{rgb}{0.90,1.0,0.90}
\definecolor{alicered}{rgb}{1.0, 0.90, 0.90}
\definecolor{aliceblue}{rgb}{0.93, 0.93, 1.0}
\newcommand{\bandzero}{\rowcolor{gray!5}}

\definecolor{lightYellow}{HTML}{FFF8DC}  
\definecolor{cornellred}{rgb}{0.7, 0.11, 0.11}
\definecolor{cadmiumgreen}{rgb}{0.0, 0.42, 0.24}
\definecolor{darkblue}{rgb}{0.83, 0.89, 0.97}
\definecolor{Red7}{rgb}{0.941, 0.243, 0.243}
\definecolor{Green7}{RGB}{55, 178, 77} 
\definecolor{Blue9}{rgb}{0.098,0.3,0.9}
\definecolor{customyellow}{HTML}{FFFCF5}
\definecolor{customred}{HTML}{FFF9F9}
\definecolor{customblue}{HTML}{F9FCFF}
\definecolor{customgreen}{HTML}{F9FFF9}
\definecolor{ForestGreen}{rgb}{0.13, 0.55, 0.13}
\definecolor{ForestBlue}{rgb}{0.13, 0.45, 0.80}
\definecolor{ForestRed}{rgb}{0.92, 0.2, 0.2}
\newcommand{\textrd}[1]{\text{\textcolor{ForestRed}{#1}}}
\newcommand{\textgr}[1]{\text{\textcolor{ForestGreen}{#1}}}
\newcommand{\textbl}[1]{\text{\textcolor{ForestBlue}{#1}}}
\usepackage{pifont}% http://ctan.org/pkg/pifont
\newcommand{\cmark}{\ding{51}}%
\newcommand{\xmark}{\ding{55}}%

\usepackage[utf8]{inputenc} % allow utf-8 input
\usepackage[T1]{fontenc}    % use 8-bit T1 fonts
\usepackage{url}            % simple URL typesetting
\usepackage{booktabs}       % professional-quality tables
\usepackage{amsfonts}       % blackboard math symbols
\usepackage{nicefrac}       % compact symbols for 1/2, etc.
\usepackage{microtype}      % microtypography
\usepackage{xcolor}
\definecolor{linkpink}{RGB}{210, 80, 140}
\usepackage[colorlinks=true, urlcolor=linkpink, citecolor=ForestGreen]{hyperref}
\usepackage[capitalize,noabbrev]{cleveref}
\title{Binding Multiple Modalities via Multimodal Wasserstein Barycenter}

\author{
	Xiaole Tang \quad 
	Jiayi Xu  \quad
	Xiang Gu  \quad
	Yan Yang \quad
	Jian Sun\thanks{Corresponding authors.}\\ \\
 Xi'an Jiaotong University
}

\begin{document}

	\maketitle

	\begin{abstract}
		Multimodal learning beyond two modalities commonly leverages a specific modality  (e.g., text) to bind other modalities. However, how to establish a more balanced representation space that approximates shared semantics while respecting the holistic geometry of $n$-modal data remains challenging. In this work, we present BaryBind, which aims to transport the specific modality towards the Wasserstein barycenter (WB) optimized across all modalities and introduces a volumetric alignment objective to establish a unified semantic space around the WB embedding. Specifically, we project specific modalities to the WB, which minimizes the average Wasserstein distances to multimodal distributions and serves as the anchor for subsequent alignment. We then construct a barycenter simplex, whose volume is taken as a similarity metric for global alignment centered at the WB.  Experiments show that BaryBind achieves competitive performance in text-video-audio retrieval, classification, videoQA, and cross-modal generation tasks, along with robustness under modality absence and scalability to more than three modalities. \href{https://github.com/xl-tang3/barybind}{{\textcolor{linkpink}\faGithub}}
	\end{abstract}
	
	\begin{figure}[!h]
		\centering
		\includegraphics[width=1\linewidth]{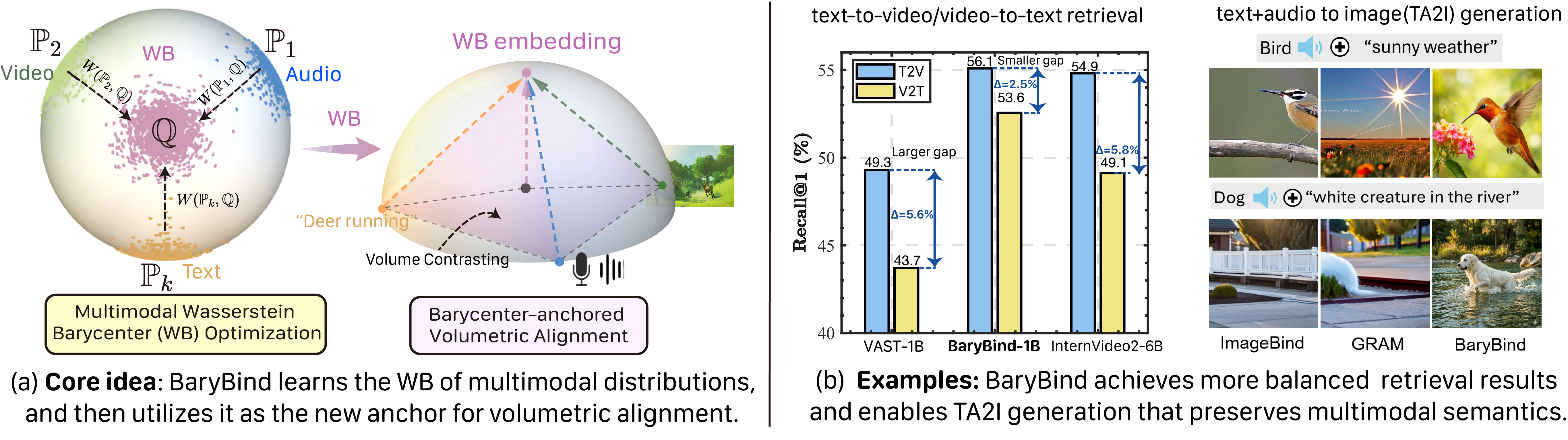}
		\caption{BaryBind learns the multimodal Wasserstein Barycenter (WB) of all modalities via WB optimization, and the WB serves as the anchor to bind modalities via volumetric alignment based on the barycenter simplex volume. Notably, BaryBind yields more balanced multimodal retrieval results and enables TA2I generation that well preserves both text and audio semantics. }
		\label{teaser}
		\vspace{-0.4cm}
	\end{figure}
	\section{Introduction}
	\vspace{-0.2cm}
	Multimodal learning \citep{baltruvsaitis2018multimodal} seeks to integrate and process heterogeneous signals from multiple modalities (e.g., vision, language, audio, depth, \textit{etc.}) to build a coherent perception of the surrounding world. Since multimodal data are heterogeneous observations of a shared underlying reality/object, recent multimodal learning methods \citep{Radford2021LearningTV, jia2021scaling} learn a shared representation from different modalities. In this field, the success of CLIP \citep{Radford2021LearningTV} in aligning unified vision–language representations via contrastive learning has sparked the adoption of contrastive losses as an appealing solution for multimodal representation learning. However, traditional contrastive losses, e.g., InfoNCE \citep{oord2018representation} and BYOL \citep{grill2020bootstrap}, are formulated in a pairwise fashion for $n=2$  cases typically arising from two modalities such as image–text \citep{Radford2021LearningTV, jia2021scaling} or audio–text \citep{guzhov2022audioclip, elizalde2023clap} scenarios. 
	
	While scaling to $n$ modalities ($n \geq 3$) poses significant challenges,  a series of recent works \citep{Zhu2023LanguageBindEV, wang2025omnibind} originating from ImageBind \citep{Girdhar2023ImageBindOE} leverage the binding property of a modality-specific anchor (e.g., image or language) to align other modalities via pairwise losses. Recently, VAST \cite{Chen2023VASTAV}, GRAM \cite{cicchetti2025gramian}, Triangle \cite{cicchetti2025triangle}, and Symile \cite{saporta2024contrasting} relied on global similarity objectives that align $n$ modalities to a designated central modality, also known as the modality-specific anchor. However, such anchors can introduce modality biases, inducing an imbalanced representation space that deviates from the shared underlying semantics, with one modality dominating the representation space. This can be observed in Fig. \ref{teaser}b, where the VAST model exhibits a large gap in Recall@1 between the text-to-video and video-to-text tasks, indicating potential asymmetry in cross-modal alignment.	 %Moreover, pairwise alignment to anchor omits the interactions among non-anchor modalities, which may insufficiently capture the global geometry within multimodal data.
	
	\begin{figure}[!t]
		\centering
		\includegraphics[width=0.7\linewidth]{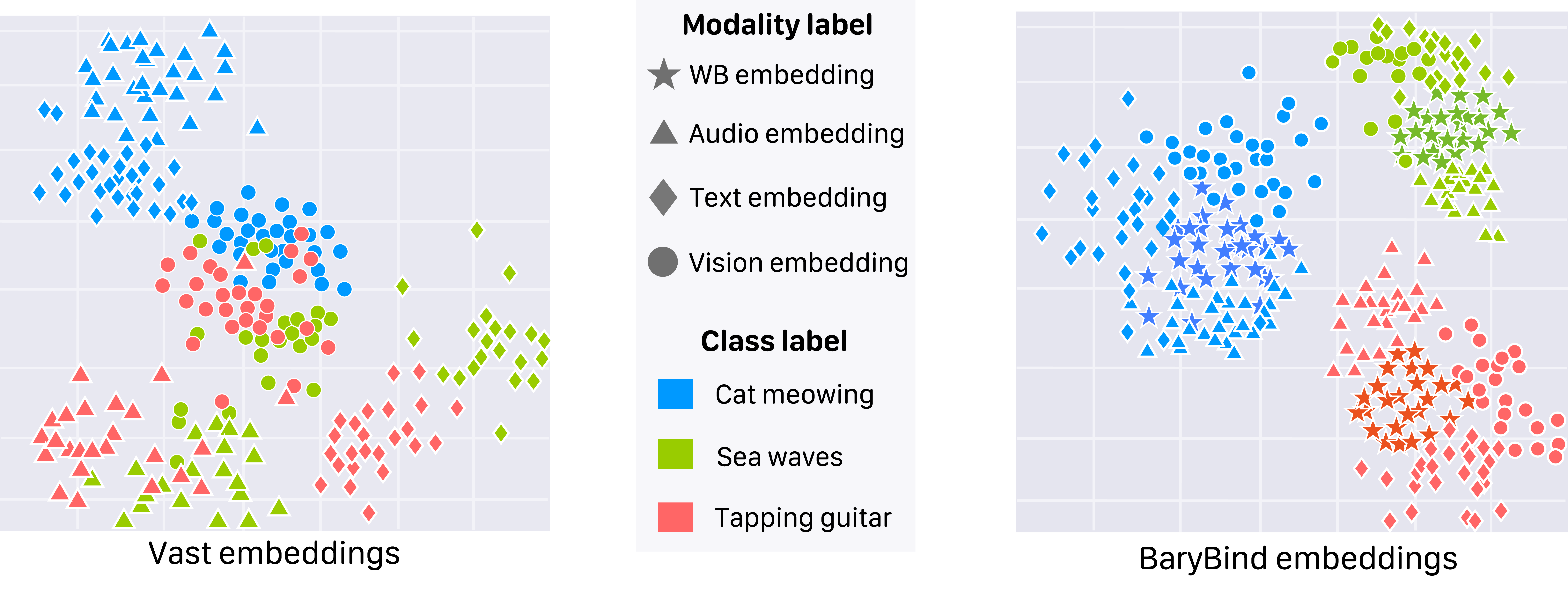}
		\caption{\textbf{The t-SNE comparison of multimodal embeddings on the zero-shot VGGSound dataset} \citep{chen2020vggsound}. Compared to VAST \citep{Chen2023VASTAV}, BaryBind induces a latent space where class clusters are clearly separated, and multimodal embeddings are grouped around the WB embeddings. }
		\label{tsne}
		\vspace{-0.3cm}
	\end{figure}
	
	To alleviate this issue, we are motivated by the view that multimodal data, though captured by heterogeneous sensors, reflect a shared underlying reality/object. Therefore, multimodal representations can be interpreted as modality-specific transformations of a shared representation that encodes semantic consensus. With this insight, we propose \textbf{BaryBind} that aligns multiple modalities to the Wasserstein Barycenter (WB) anchor, which defines a distribution that minimizes the average of Wasserstein distances to all multimodal distributions.  As shown in Fig. \ref{teaser}a, the WB is optimized across all modalities via the multimodal WB optimization process and then utilize as the anchor for subsequent alignment, aiming to mitigate modality-specific bias and encourages semantic consensus (Fig. \ref{tsne}). Based on the WB anchor, we further introduce a volumetric contrastive loss defined on the volume of a barycenter simplex, which quantifies the global alignment of $n$-modal data. By contrasting the simplex's volume (Fig. \ref{teaser}a), BaryBind aligns modalities to the barycenter while preserving the holistic geometry within $n$-modal data. Interestingly, with the barycenter-based alignment, we observe that BaryBind substantially alleviates the imbalance across modalities, reducing the T2V/V2T retrieval gap by 3.1\%  compared to VAST \citep{Chen2023VASTAV} (Fig. \ref{teaser}b), and induces a latent space where modalities of each class are consistently clustered around the WB embeddings (Fig. \ref{tsne}).

	The main contributions are highlighted as follows:
	\begin{itemize}[leftmargin=*,itemsep=0mm]
		\item \textbf{Method:} We propose BaryBind, which transports the specific anchor to the WB that mitigates modality bias as the geometric semantic center. Multimodal embeddings are aligned to the WB via a volumetric alignment objective, enabling balanced multimodal understanding.
		\item \textbf{Theory:} We develop an optimal transport (OT) duality-based optimization framework and error bounds for approximating the Wasserstein barycenter of multimodal representations. Meanwhile, we establish a barycenter simplex, whose volume serves as a global similarity metric and extends the insights of measuring $n$-modal alignment. 
		\item \textbf{Performance:} BaryBind achieves balanced state-of-the-art performance in cross-modal retrieval, generation, multimodal classification, and videoQA tasks across diverse datasets. Notably, BaryBind shows unique robustness under modality absence and scalability to three or more modalities.
	\end{itemize}
	\section{Related Works}
	
	\textbf{Multimodal representation learning.} Multimodal representation learning seeks to align heterogeneous modalities into a shared semantic space. CLIP~\citep{Radford2021LearningTV} initiates this paradigm with image–text contrastive learning, followed by audio extensions such as AudioCLIP~\citep{guzhov2022audioclip}, CLAP~\citep{elizalde2023clap}, and LAION-CLAP~\citep{laionclap2023}. WavCaps~\citep{mei2024wavcaps} constructs a large-scale audio captioning dataset to support audio–language retrieval. To scale beyond two modalities, ImageBind~\citep{Girdhar2023ImageBindOE} introduces a pivot-based strategy, aligning each modality to the vision anchor. LanguageBind~\citep{Zhu2023LanguageBindEV},  UniBind~\citep{lyu2024unibind}, GRAM \citep{cicchetti2025gramian},  Triangle \citep{cicchetti2025triangle}, and LCO-EMB \citep{xiaoscaling}  extend alignment to language anchors, adopting unique volume-based contrastive losses.  OmniBind~\citep{wang2025omnibind} aligns pre-trained unimodal experts via pairwise losses. VAST~\citep{Chen2023VASTAV} fuses modalities into a shared space but still relies on text-centered supervision. 
	
	In contrast, BaryBind alleviates the dominance of the specific anchor modality by explicitly transforming it into the shared WB, which is optimized as the barycenter of all multimodal distributions and acts as the refined anchor. It further leverages a barycenter-anchored volumetric contrastive loss that enables global alignment to the WB embedding, thereby establishing a more balanced semantic space centered at the WB embeddings.

	\textbf{Wasserstein barycenter.} The WB~\citep{agueh2011barycenters, beier2023multi} defines an averaging distribution that minimizes the weighted sum of Wasserstein distances \cite{villani2009optimal,chemseddine2025conditional,piening2026novel} to input measures, preserving mass structure and OT-grounded geometry. This formulation has shown effectiveness on heterogeneous supports and has been applied in generative modeling~\citep{cuturi2014fast} and domain adaptation~\citep{bonneel2015sliced}. Recent works estimate high-dimensional barycenters via deep dual formulations, including ICNN-based cycle-consistent models~\citep{korotin2021continuous}, neural OT maps~\citep{kolesov2024estimating, tang2025baryir, tang2026learning,tang2025degradation}, and energy-guided potentials~\citep{kolesov2024energy}. We extend this perspective to construct a barycenter-aligned joint space for multimodal alignment, bridging optimal transport with multimodal representation learning.
	
	\section{Method}
	\begin{figure*}[!t]
		\centering
		\includegraphics[width=1.0\linewidth]{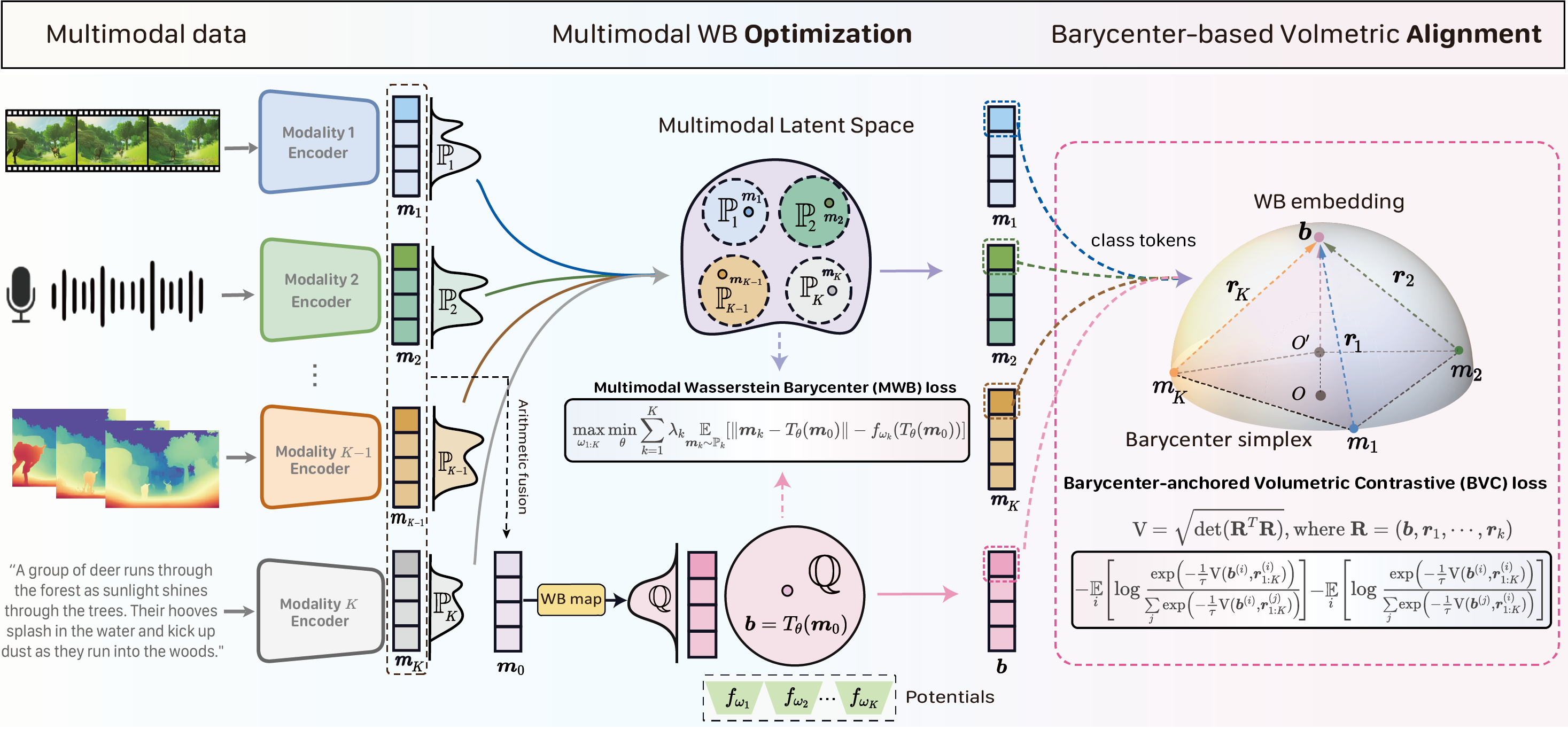}
		\caption{\textbf{BaryBind establishes unified representation space through two processes:} 1) The \textit{WB optimization} process transforms modality-specific initializer $\boldsymbol {m_0}$ into the WB anchor $\boldsymbol b$, in which the WB map $T_\theta$ is optimized by minimizing the average of Wasserstein distances to all multimodal distributions $\mathbb P_k$. 2) The \textit{volumetric alignment} process utilizes the WB anchor to induce a barycenter simplex, whose volume is contrasted to achieve global alignment to the WB embedding.}
		\label{model}
		\vspace{-0.2cm}
	\end{figure*}
	
	\textbf{Overview.} In this work, we introduce BaryBind that aligns multiple modalities to the WB and suggest a volumetric alignment objective to establish a unified representation space. The overview is presented in Fig. \ref{model}, where the WB optimization process (\S\ref{MWBS}) learns a transformation from specific modality to the WB, which serves as the alignment center. The volumetric alignment process (\S\ref{volume}) enables global alignment to the WB anchor by contrasting the volume of the barycenter simplex.
	%	
	%	
	%	Specifically, as shown in Fig. \ref{model}, BaryBind first constructs the multimodal WB space that aligns features from multimodal latent space via the multimodal WB loss (\S\ref{MWBS}),  in which an MLP is learned to transport the anchor features to the WB space as WB embeddings. We then construct a barycenter simplex defined by the WB embeddings and modality-to-barycenter gap vectors, whose volume quantifies the degree of $n$-modality alignment (\S\ref{volume}). Based on the simplex, we introduce a barycenter-anchored volumetric contrastive loss, which encourages high-order multimodal alignment to the WB space while reducing inter-modality gaps (\S\ref{contrast}).
	\subsection{Preliminaries}
	\textbf{Notation.} 
	Let $\bar{K} = \{1, \dots, K\}$ for some $K \in \mathbb{N}$. For a sequence $e_0, \dots, e_K$, we denote by $e_{1:K}$ the tuple $(e_0, \dots, e_K)$. Let $\mathcal{X} \subset \mathbb{R}^d$, $\mathcal{Y} \subset \mathbb{R}^{d'}$, and $\mathcal{X}_k \subset \mathbb{R}^{d_k}$ be compact subsets of Euclidean space. Denote by $\mathcal{C}(\mathcal{X})$ the space of continuous real-valued functions on $\mathcal{X}$, and by $\mathcal{P}(\mathcal{X})$ the set of probability measures supported on $\mathcal{X}$. Given $\mathbb{P} \in \mathcal{P}(\mathcal{X})$ and $\mathbb{Q} \in \mathcal{P}(\mathcal{Y})$, we write $\Pi(\mathbb{P}, \mathbb{Q})$ for the set of transport plans between them, i.e., all joint distributions on $\mathcal{X} \times \mathcal{Y}$ whose marginals are $\mathbb{P}$ and $\mathbb{Q}$. The notation $\langle \cdot, \cdot \rangle$ denotes the cosine similarity over features.
	
	\textbf{Optimal transport.} Given two distributions $\mathbb{P} \in \mathcal{P}(\mathcal{X})$ and $\mathbb{Q} \in \mathcal{P}(\mathcal{Y})$, along with a cost function $c : \mathcal{X} \times \mathcal{Y} \to \mathbb{R}_+$, the classic optimal transport (OT) problem \citep{kantorovich1942translocation} aims to find a joint distribution $\pi \in \Pi(\mathbb{P}, \mathbb{Q})$ that minimizes the expected transport cost:
	\begin{align}
		\text{OT}_c(\mathbb{P},\mathbb{Q})\triangleq\inf_{\pi\in\Pi(\mathbb P,\mathbb Q)}\mathbb E_{(x,y)\sim\pi}\left[c(x,y)\right].
		\label{Kon}
	\end{align}
	The specific choice of $c(x,y)=\|x-y\|$ yields $W(\mathbb P,\mathbb Q)=\inf_{\pi\in\Pi(\mathbb P,\mathbb Q)}\mathbb E_{(x,y)\sim\pi}\|x-y\|$, known as the Earth-Mover or Wasserstein distance.
	
	\textbf{Wasserstein barycenter (WB).} Given distributions $\mathbb P_k\in\mathcal P(\mathcal X_k)$ for $k\in\bar K$ and a vector $\lambda\in\mathbb R^{K+1}$ of non-negative weights $\lambda_k$ summing to 1, the WB problem seeks the distribution $\mathbb Q$ that minimizes the weighted sum of Wasserstein distances to the fixed marginals $\mathbb P_{1:K}$:
	\begin{align}
		\inf_{\mathbb Q\in\mathcal P(\mathcal Y)}\sum_{k=1}^{K}\lambda_kW(\mathbb P_k, \mathbb Q). \label{WB}
	\end{align}
	We apply the WB formulation to the multimodal latent space $\mathcal{M}$, where the encoded features of the $k$-th modality lie in a subspace $\mathcal{M}_k$ and follow a distribution $\mathbb{P}_k$. The WB approximates the geometric consensus of semantics as it models the ``closest'' distribution to all multimodal distributions. 
	
	\subsection{Multimodal Wasserstein Barycenter Optimization}
	\label{MWBS}
	Let $\mathbb{P}_k$ be the distribution of encoded features $\boldsymbol m_k\in\mathcal M_k\subset\mathcal M$ for modality $k \in \bar K$, defined in the multimodal latent space $\mathcal{M}\subset\mathbb R^D$. $\boldsymbol m_0$ denotes the anchor modality feature. The WB space is defined as $\mathcal M_B:=\text{supp}(\mathbb Q)$ where $\mathbb Q$ denotes the WB distribution and  $\mathcal M_B$ contains the barycenter features $\boldsymbol b$. Given the distributions $\mathbb{P}_{1:K}$, our goal is to establish the barycenter space $(\mathcal M_B, \mathbb Q)$ and use it as the joint representation space for multimodal alignment. Based on the WB formulation (\ref{WB}) over multimodal latent space, the multimodal WB problem can be written as
	\begin{align}
		\mathcal L_{\mathrm{MWB}}^*=\inf_{\mathbb Q\in\mathcal P(\mathcal M_B)}\sum_{k=1}^{K}\lambda_kW(\mathbb P_k, \mathbb Q). \label{MWB}
	\end{align} 
	This formulation induces a barycenter distribution that minimizes the average of Wasserstein distances to all modalities. However, directly optimizing (\ref{MWB}) is highly intractable. Therefore, we deduce a new dual reformulation (see Appendix \ref{pf1}) with sup-inf objective as follows.
	
	\begin{proposition}[\textbf{Duality of multimodal WB problem}]The infimum  value $\mathcal{L}^*_{\mathrm{MWB}}$ of the multimodal WB problem (\ref{MWB}) can be expressed as	
		\label{prop1}
		\begin{align}
			\mathcal L^*_{\mathrm{MWB}}=\sup\limits_{\sum_k\!\lambda_k f_k = 0}~\inf_{\substack{\\[0.1ex]\mathbb Q\in\mathcal P(\mathcal M_B)}}\sum_{k=1}^K\lambda_k\mathop{\mathbb E}_{\substack{\boldsymbol m_k\sim\mathbb P_k\\\boldsymbol b\sim\mathbb Q}}\big[\|\boldsymbol m_k-\boldsymbol b\|-f_k(\boldsymbol b)\big],
			\label{duality}
		\end{align}
	\end{proposition}
	where the supremum is taken over the potentials $f_k:\mathcal M_B\rightarrow\mathbb R$. We aim to learn the distribution $\mathbb Q$ by sampling WB embedding $\boldsymbol b=T(\boldsymbol {m_0})$ via a trainable WB map $T_{\theta}(\cdot):\mathcal M_0\rightarrow \mathcal M_B$. Here, $\boldsymbol m_0$ represents the WB initializer, empirically selected as the arithmetic mean of all available multimodal features, i.e., $\boldsymbol{m_0} = \bar{\boldsymbol m}_i$ and can also be chosen as any modality (see ablation in \textbf{Appendix \ref{sec:anchor_rationale}}). We parameterize $T_\theta$ using an MLP. To ensure the congruence $\sum_k\!\lambda_k f_k = 0$ \citep{li2020continuous}, we parameterize the potentials $f_{\omega_k}$ as $g_{\omega_k}-\sum_{i=1}^K\lambda_ig_{\omega_i}$ with two-layer MLPs $g_{\omega_k}:\mathbb R^D\rightarrow\mathbb R$, which is a common trick used in \citep{li2020continuous,kolesov2024estimating,kolesov2024energy}. 
	
	\textbf{Multimodal Wasserstein barycenter (MWB) loss}. With this parameterization, we rewrite (\ref{duality}) as a max-min objective of the MWB loss $\mathcal L_{\mathrm{MWB}}$, which can be optimized to compute the barycenter map $T_\theta$ for approximating the WB:
	\begin{align}
		\max_{\omega_{1:K}}\min_{\theta}\bigg\{\mathcal L_{\mathrm{MWB}}(\omega_{1:K},\theta)\triangleq\sum_{k=1}^K \lambda_k\mathop{\mathbb E}_{\boldsymbol m_k\sim\mathbb P_k}\big[\|\boldsymbol m_k-T_\theta(\boldsymbol m_0)\|-f_{\omega_k}(T_\theta(\boldsymbol m_0))\big]\bigg\}.
		\label{Lbary}
		%	\mathcal L_{bary}(\omega_k,\theta)=\underbrace
	\end{align}
	To solve the problem (\ref{Lbary}), we train $T_\theta$ and $f_{\omega_{1:K}}$ by alternately maximizing over $\omega_{1:K}$ and minimizing over $\theta$ in the MWB loss, in which we estimate the expectation using mini-batch data at each training step. During the optimization, all modalities jointly update the barycenter, ensuring that no single modality dominates and that the modality bias is reduced. Then the WB embedding is computed as $\boldsymbol b= T_\theta(\boldsymbol m_0)$, serving as the  anchor for subsequent alignment. 
	
	We also establish the error bounds for the map $T$ with the following  simplified notations:
	\begin{align}&\mathcal F(f_{1:K},T):=\mathcal L_{\mathrm{MWB}}(f_{1:K},T),~~~
		\mathcal L(f_{1:K}):=\inf_{T:\mathcal M\rightarrow\mathcal M_B}\mathcal F(f_{1:K},T),~~~\mathcal L^*:=\mathcal L_{\mathrm{MWB}}^*. \label{func2} \end{align}
	
	\begin{theorem}[\normalfont{Error analysis via duality gaps for the barycenter distribution}] Let $C_k$ be any transport costs. Assume that the maps \label{error}
		$\boldsymbol b \mapsto C_k(\boldsymbol m_k, \boldsymbol b) - \widehat{f}_k(\boldsymbol b)$ 
		are $\beta$-strongly convex for $\boldsymbol m_k \in \mathcal{M}_k$, $k \in \{0, \dots, K\}$.
		Consider the duality gaps for an approximate solution $(\widehat f_{1:K}, \widehat T)$:
		\begin{align}
			\mathcal{E}_1(\widehat{f}_{1:K}, \widehat{T}) 
			\triangleq 
			\mathcal{F}(\widehat{f}_{1:K}, \widehat{T})
			-\mathcal{L}(\widehat{f}_{1:K}),~~~~~
			\mathcal{E}_2(\widehat{f}_{1:K})
			\triangleq
			\mathcal{L}^* - \mathcal{L}(\widehat{f}_{1:K}),
		\end{align}
		Then the following inequality holds:
		\begin{equation}
			W^{2}\!\left(\widehat T_{\#}\mathbb P_0, \mathbb Q^*\right)
			\leq \frac{4}{\beta}\big(\mathcal E_1 + \mathcal E_2\big).
		\end{equation}
	\end{theorem}
	\noindent The proof is provided in the \textbf{Appendix \ref{pf2}}. This theorem ensures that the Wasserstein distance between the estimated distribution $\widehat{T}_{\#}\mathbb{P}_0$ and the true barycenter $\mathbb{Q}^*$ is upper-bounded by the sum of these two errors. In our setting with Euclidean cost, this local condition can hold when the learned neural potentials are sufficiently smooth. In practice, regularization (e.g., weight decay) can improve the smoothness of neural potentials, and our experiments demonstrate the effectiveness of the proposed Wasserstein barycenter solver. Our experiments also demonstrate that the WB solver performs well with Euclidean cost.

	\subsection{Barycenter-anchored Volumetric Alignment} 
	\label{volume}
	
	To bind modalities to the WB anchor while preserving the global structures of \(K\)-dimensional multimodal data, we introduce a geometric structure called the \textit{barycenter simplex}. As illustrated in Fig.~\ref{model}, the simplex takes the WB embedding \(\boldsymbol{b}\) as the apex and is spanned by the vectors from the origin to \(\boldsymbol{b}\) and the modality-to-barycenter gap vectors \(\boldsymbol{r}_k = \boldsymbol{b} - \boldsymbol{m}_k\) for all modalities \(\boldsymbol{m}_k\) (\(k \geq 1\)). Owing to its unique composition, the simplex volume quantifies two aspects of multimodal features: (1) the global alignment to the barycenter, and (2) the inter-modality discrepancy across modalities.

	Given the WB embedding $\boldsymbol b$ (i.e., the vector from the origin to the barycenter) and gap vectors $\{\boldsymbol r_k\}_{k=1}^K$ (i.e., the vectors from each modality to the barycenter), we define the matrix $\mathbf R=(\boldsymbol b,\boldsymbol r_1,\cdots,\boldsymbol r_K)$ as the set of vectors spanning the barycenter simplex, whose square of volume can be computed according to \citet{Gantmacher1959matrix} as:
	\begin{align}
		\label{vol}
		\text{V}^2(&\boldsymbol b,\boldsymbol r_{1:K}):=\text{V}^2(\boldsymbol b,\boldsymbol r_1,\cdots,\boldsymbol r_K)=\det(\mathbf R^T \mathbf R)
		=\begin{vmatrix}
			\langle\boldsymbol{b},\boldsymbol{b}\rangle & \langle\boldsymbol{b},\boldsymbol{r}_1\rangle & \cdots & \langle\boldsymbol{b},\boldsymbol{r}_K\rangle \\
			\langle\boldsymbol{r}_1,\boldsymbol{b}\rangle & \langle\boldsymbol{r}_1,\boldsymbol{r}_1\rangle & \cdots & \langle\boldsymbol{r}_1,\boldsymbol{r}_K\rangle \\
			\vdots & \vdots & \ddots & \vdots \\
			\langle\boldsymbol{r}_K,\boldsymbol{b}\rangle & \langle\boldsymbol{r}_K,\boldsymbol{r}_1\rangle & \cdots & \langle\boldsymbol{r}_K,\boldsymbol{r}_K\rangle
		\end{vmatrix}.
	\end{align}
	
	%	\det(\mathbf R^T \mathbf R)=\begin{vmatrix}
		%		\langle\boldsymbol{b},\boldsymbol{b}\rangle & \langle\boldsymbol{b},\boldsymbol{r}_1\rangle & \cdots & \langle\boldsymbol{b},\boldsymbol{r}_K\rangle \\
		%		\langle\boldsymbol{r}_1,\boldsymbol{b}\rangle & \langle\boldsymbol{r}_1,\boldsymbol{r}_1\rangle & \cdots & \langle\boldsymbol{r}_1,\boldsymbol{r}_K\rangle \\
		%		\vdots & \vdots & \ddots & \vdots \\
		%		\langle\boldsymbol{r}_K,\boldsymbol{b}\rangle & \langle\boldsymbol{r}_K,\boldsymbol{r}_1\rangle & \cdots & \langle\boldsymbol{r}_K,\boldsymbol{r}_K\rangle
		%	\end{vmatrix}.
	
	\textbf{Barycenter-anchored volumetric contrastive (BVC)  loss}. We propose the barycenter-anchored volumetric contrastive (BVC) loss that contrasts the volume of the barycenter simplex. Given multimodal embeddings \(\{\boldsymbol m_{1}^{(i)}, \dots, \boldsymbol m_{K}^{(i)}\}_{i=1}^B\), where $B$ is the batch size and \(\boldsymbol m_{0}^{(i)}\) is the anchor, the WB anchor is computed by \(\boldsymbol{b}^{(i)} = T_\theta(\boldsymbol m_{0}^{(i)})\), and the modality-to-barycenter gap vectors are defined as \(\boldsymbol {r}_{1:K}^{(i)} = \{\boldsymbol{r}_{1}^{(i)}, \dots, \boldsymbol{r}_{K}^{(i)}\}\). We treat each \((\boldsymbol{b}^{(i)},\boldsymbol {r}_{1:K}^{(i)})\) as a \emph{positive pair}, and construct two sets of \emph{negative pairs}: (i) by fixing \(\boldsymbol{b}^{(i)}\) and pairing it with \(\boldsymbol {r}_{1:K}^{(j)}\) \(j \ne i\), and (ii) by fixing \(\boldsymbol {r}_{1:K}^{(i)}\) and pairing it with \(\boldsymbol{b}_j\) for \(j \ne i\). 
	
	The BVC loss contrasts small volumes for positive pairs against large volumes for negative ones:
	\begin{align}
		\mathcal{L}_{\mathrm{BVC}} = 
		-\mathop{\mathbb{E}}_{i} \left[
		\log \frac{
			\exp\left( -\mathrm{V}(\boldsymbol{b}^{(i)}, \boldsymbol {r}_{1:K}^{(i)}) / \tau \right)
		}{
			\mathop{\sum}_j \exp\left( -\mathrm{V}(\boldsymbol{b}^{(i)}, \boldsymbol {r}_{1:K}^{(j)}) / \tau \right)
		}+\log \frac{
			\exp\left( -\mathrm{V}(\boldsymbol{b}^{(i)}, \boldsymbol {r}_{1:K}^{(i)}) / \tau \right)
		}{
			\mathop{\sum}_j \exp\left( -\mathrm{V}(\boldsymbol{b}^{(j)}, \boldsymbol {r}_{1:K}^{(i)}) / \tau \right)
		}\right]
		,
		\label{eq:vcl}
	\end{align}
	where \(\mathrm{V}(\boldsymbol{b}, \boldsymbol{r}_{1:K})\) denotes the volume of the barycenter simplex formed by the WB and modality-to-barycenter gap vectors, whose square is computed as (\ref{vol}).	The BVC loss binds multiple modalities to the shared WB embedding, ensuring a more holistic and accurate alignment of $n$-modal data.
	
	\subsection{Training Loss}
	To further guide the training process following VAST \citep{Chen2023VASTAV}, we employ a data-anchor matching (DAM) loss that distinguishes matched and mismatched data pairs:
	\begin{align}
		\mathcal L_{\mathrm{DAM}}=\mathbb E_{(\boldsymbol b,\boldsymbol m_{1:K})\sim (\mathbb Q,\mathbb P_{1:K})}[y\log p_m(\boldsymbol b,\boldsymbol m_{1:K})+(1-y)\log (1-p_m(\boldsymbol b,\boldsymbol m_{1:K}))].
		\label{bdm}
	\end{align}
	To produce binary predictions $p_m$, the modality-specific anchor data are fed into the anchor encoder and integrated with the concatenated non-anchor features $\boldsymbol m_{1:K}$ by cross-attention layers. Then the output feature is passed through a two-layer MLP to generate the output probability.
	
	The overall pre-training objective $\mathcal{L}$ combines the three losses, $\mathcal{L}_{\text{MWB}}$, $\mathcal{L}_{\text{BVC}}$, and $\mathcal{L}_{\text{DAM}}$, as follows:
	\begin{align}
		\mathcal{L} := \mathcal{L}_{\text{MWB}} + \alpha_1\mathcal{L}_{\text{BVC}} + \alpha_2\mathcal{L}_{\text{DAM}}.
		\label{tLoss}
	\end{align}
Particularly, the potentials $f_{\omega_{1:K}}$ in $\mathcal{L}_{\text{MWB}}$ are optimized via maximization, and are trained in an alternating manner against the minimization of the remaining networks. The training pipeline in summarized in Algorithm \ref{algo} in the Appendix \ref{algorithm}.

	\begin{figure}[!h]
		\centering
		\includegraphics[width=0.9\linewidth]{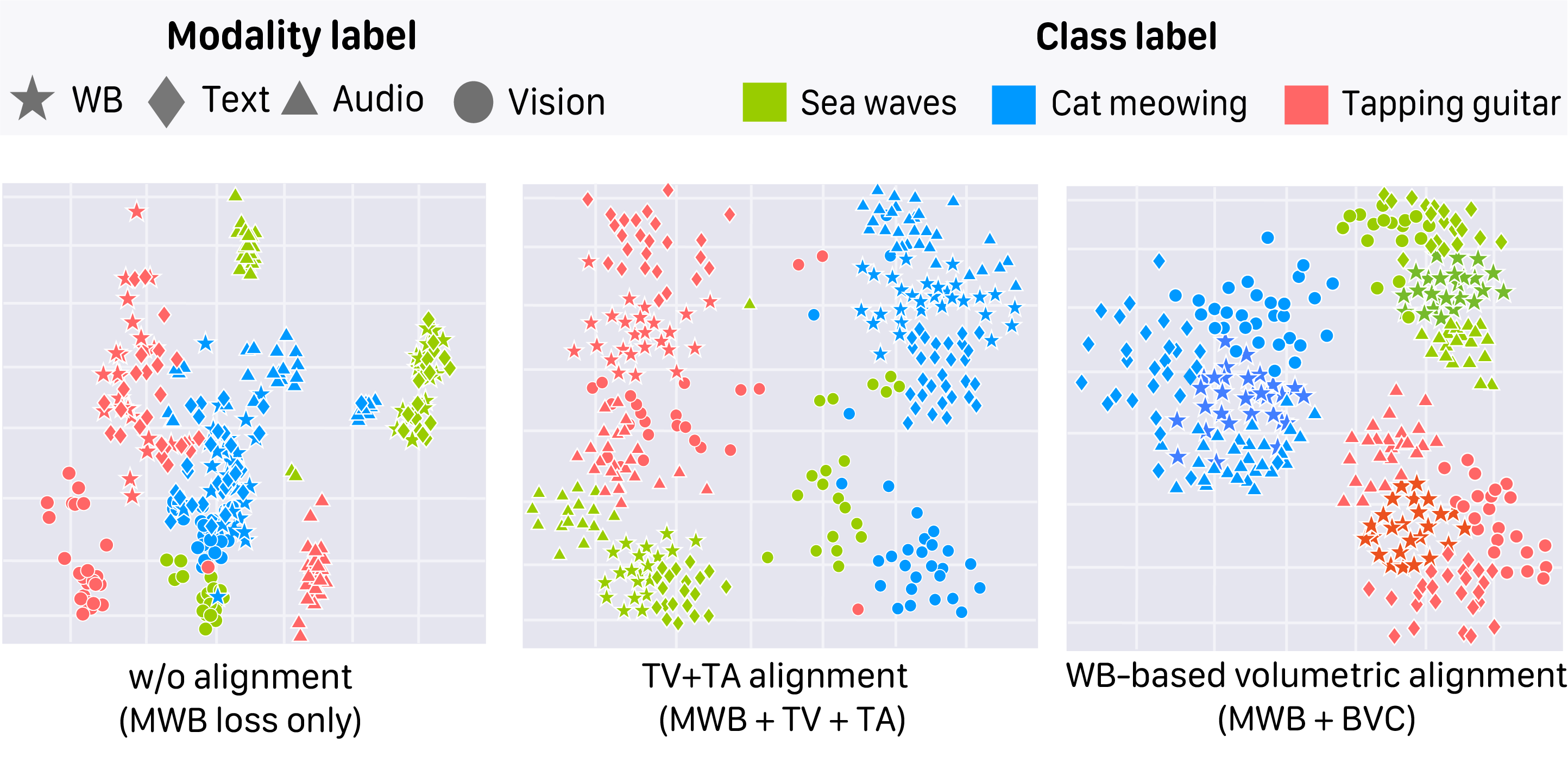}
		\caption{\textbf{Visualization of multimodal embeddings} before alignment (left) and after applying TV + TA cosine-based alignment (middle) and the proposed BVC alignment (right). The BVC loss promotes convergence of embeddings towards the WB anchors, forming compact clusters around the WB embeddings, which highlights improved balanced multimodal alignment.}
		\label{abalign}
	\end{figure}	
	\textbf{Visualization of multimodal embeddings before and after alignment}.
	To intuitively illustrate the effect of our barycenter-based alignment, we visualize the embeddings of three categories (cat meowing, sea waves, tapping guitar) across audio, text, and vision modalities from VGGSound \citep{chen2020vggsound}, as shown in Fig. \ref{abalign}. Without alignment (left), embeddings of different modalities are scattered, exhibiting notable modality gaps within the same class. Incorporating pairwise alignment (middle) improves class clustering but still reveals modality separations. In contrast, with our BVC loss (right), embeddings of all modalities converge around their class-wise Wasserstein barycenters, leading to compact intra-class structures and reduced inter-modality discrepancies. This clearly demonstrates that the volumetric constraint effectively promotes balanced alignment around the WB embeddings.

	\section{Experiments}
	\label{exp}
	
	%\subsection{Comparison with State-of-the-Art Methods}
	\subsection{Experimental Setup} 
	\begin{wraptable}[21]{R}{0.5\textwidth}
		\vspace{-0.3cm}
		\centering\small
		\caption{
			\textbf{Zero-shot video/audio classification} results on VGGSound5K. Results from our baseline \sethlcolor{customyellow}\hl{VAST} and the proposed \sethlcolor{customblue}\hl{BaryBind}, and the performance gains are highlighted accordingly. 
		}
		\resizebox{0.5\textwidth}{!}{
			\begin{tabular}{l c ll}

				\toprule
				Method &Modality& Acc@1&  Acc@5\\
				\midrule
				ImageBind \citep{Girdhar2023ImageBindOE} & A & 31.6 & 58.7 \\
				ImageBind \citep{Girdhar2023ImageBindOE} & V & 37.9 & 65.9 \\
				LanguageBind \citep{Zhu2023LanguageBindEV} & A & 34.1  &62.8 \\
				LanguageBind \citep{Zhu2023LanguageBindEV} & V & 39.6 & 64.5 \\
				{GRAM} \citep{cicchetti2025gramian} &V& 41.6&74.3\\
				{GRAM} \citep{cicchetti2025gramian} &A+V& 42.3&76.4\\
				{Triangle} \citep{cicchetti2025triangle} &A+V& 44.8&80.0\\
				OmniBind \citep{wang2025omnibind} &A&41.7&70.8\\
				OmniBind \citep{wang2025omnibind} &V&45.4&73.2\\
				OmniBind \citep{wang2025omnibind} &A+V&46.2&76.2\\
				\midrule
				\rowcolor{customyellow}	VAST \citep{Chen2023VASTAV} &A&40.3&71.7\\
				\rowcolor{customyellow}	VAST \citep{Chen2023VASTAV} &V&46.3&72.7\\
				\rowcolor{customyellow}	VAST \citep{Chen2023VASTAV} &A+V&48.1&79.6\\
				\midrule
				\rowcolor{customblue} BaryBind & A&45.7 \textbf{\textbl{\tiny\!(+5.4)}} &75.2 \textbf{\textbl{\tiny\!(+3.5)}}\\ 
				\rowcolor{customblue} BaryBind & V&48.3 \textbf{\textbl{\tiny\!(+2.0)}}&76.4 \textbf{\textbl{\tiny\!(+3.7)}}\\ 
				\rowcolor{customblue} BaryBind & A+V &\textbf{55.6 \textbl{\tiny\!(+7.5)}}&\textbf{83.4 \textbl{\tiny\!(+3.8)}}\\ 
				\bottomrule
				
			\end{tabular}
					\label{tab:VGG}
		}
		\vspace{-0.3cm}
	\end{wraptable}
	
	\label{4.1}
	We adopt VAST~\citep{Chen2023VASTAV} as the backbone, with BERT-B for text, BEATs for audio, and EVA-CLIP-ViT-G~\citep{sun2023eva} for visual encoding. In total, the model comprises approximately 1B parameters (see \textbf{Appendix \ref{overhead}} for computational overhead analysis). Unlike VAST, we discard its modality fusion layers and introduce lightweight MLPs for barycenter optimization. Following the setting of GRAM \cite{cicchetti2025gramian} and Triagnle \cite{cicchetti2025triangle}, we continue pretraining based on VAST using our proposed loss functions on the VAST150k dataset~\citep{Chen2023VASTAV}. VAST \cite{Chen2023VASTAV}, GRAM \cite{cicchetti2025gramian}, Triangle \cite{cicchetti2025triangle}, and PMRL \cite{liu2026principled} adopt the same architectures and training strategy for a fair comparison.  $\tau$ is 0.07 and the trade-off parameters in (\ref{tLoss}) are set as $\alpha_1=1,\alpha_2 = 0.1$ (see Appendix \ref{sec:sens} for sensitivity analysis). The barycenter weights $\lambda_k$ ($k\in\bar K$) are uniformly set as $1/K$, where $n$ is the number of modalities. The compared methods are implemented under identical experimental conditions.

	We use benchmarks spanning diverse modalities to evaluate the BaryBind's capability in multimodal understanding across retrieval and classification tasks. These benchmarks include: (i) three-modality datasets DiDeMo~\citep{anne2017localizing} and ActivityNet~\citep{caba2015activitynet}, where video serves as the primary modality while audio and text provide auxiliary cues; (ii) four-modality datasets MSR-VTT~\citep{xu2016msr} and VATEX~\citep{wang2019vatex}, covering video (V), audio (A), text (T), subtitles (S) and we additionally introduce depth (D) modality, which is derived using ChronoDepth~\citep{shao2025learning} and integrated via an additional lightweight head attached to the vision encoder; and (iii) audio-centered dataset  VGGSound~\citep{chen2020vggsound}, where audio plays the dominant role and complementary information is also available in visual and textual forms. T-VAS denotes that the WB embedding is derived from the text, while the concatenated VAS embeddings serve as the condition. More details, evaluations on multimodal tasks, covering \textit{video-QA} and \textit{cross-modal generation}, and ablation studies are provided in \textbf{Appendix \ref{details}}.

	\subsubsection{Zero-shot video/audio classification}

	We first evaluate BaryBind on VGGSound5K to assess its multimodal understanding ability in the zero-shot setting, particularly when jointly modeling video and audio signals. As shown in Table~\ref{tab:VGG}, BaryBind achieves the best top-1 and top-5 accuracy across all modality configurations, which verifies its strong generalization ability across modalities.  It reaches (55.6\%/83.4\%) on the classification integrating video and audio, significantly outperforming the VAST baseline (48.1\%/79.6\%) and other competitors such as OmniBind (46.2\%/76.2\%), GRAM (42.3\%/74.5\%). The substantial gains with A+V highlight the enhanced holistic semantics understanding, enabled by the BVC loss that preserves the global geometry of $n$-modal data.
	
	On the other hand, BaryBind improves the performance of the weaker audio-only modality classification (A) from 40.3\% (VAST) to 45.7\% for Acc@1, indicating its ability to alleviate under-optimization of weaker modalities. These results confirm that BaryBind learns a more balanced semantic space via our barycenter-based alignment strategy.

	\begin{table*}[!t]
		\centering
		\caption{\textbf{Zero-shot} text-to-video (T2V) and video-to-text (V2T) retrieval results in terms of Recall at 1 score (R@1). Results from our baseline \sethlcolor{customred}\hl{VAST} and  \sethlcolor{customgreen}\hl{BaryBind} are highlighted accordingly.}
		\label{table:videoText}
		\setlength{\tabcolsep}{11pt}
		\renewcommand{\arraystretch}{1}
		\resizebox{\textwidth}{!}{
			\begin{tabular}{lc ll llllll}
				\toprule
				\multirow{2}{*}{Methods}&\multirow{2}{*}{Modality} & \multicolumn{2}{c}{MSR-VTT} & \multicolumn{2}{c}{DiDeMo} & \multicolumn{2}{c}{ActivityNet}  & \multicolumn{2}{c}{VATEX}   \\ 
				\cmidrule(lr){3-4}
				\cmidrule(lr){5-6}
				\cmidrule(lr){7-8}
				\cmidrule(lr){9-10}
				&  & T2V    & V2T   & T2V     & V2T  & T2V    & V2T  & T2V    & V2T    \\ \midrule
				X-CLIP \citep{ma2022x} & T-V & 46.1    & 46.8  & 45.2         &42.3     & 44.3     & 42.6 & - & - \\
				ImageBind \citep{Girdhar2023ImageBindOE}  & T-V & 36.8    &  -  & -        & -     & -    & -  & -&-\\
				ViCLIP \citep{wang2024internvid}  & T-V & 42.4  & 41.3      & 18.4      & 27.9     & 15.1    & 24.0 & - & - \\
				
				VideoPrism-b \citep{Zhao2024VideoPrismAF}  & T-V & 51.4     & 50.2   & -    & -  & 49.6      &  47.9 & 62.5 & 77.1\\
				LanguageBind \citep{Zhu2023LanguageBindEV}  & T-V & 44.8     & 40.9   & 39.9    & 39.8  & 41.0 & 39.1 & - & - \\
				InternVL \citep{chen2024internvl}&T-V&46.3&42.4&43.7&42.2&45.1&42.4&66.8&69.3\\
				OmniBind \citep{wang2025omnibind} &T-V&47.4&45.2&43.5&42.6&44.3&40.8&-&-\\
				NarVid \citep{hur2025nar} &T-V&51.8&50.3&52.4&50.5&51.8&46.6&73.8&76.3\\
				Video-ColBERT \citep{Reddy_2025_CVPR} &T-V&51.9&48.8&51.7&50.1&52.7&47.8&72.4&73.7\\
					\midrule
				\rowcolor{customred} VAST \citep{Chen2023VASTAV}  & T-VA & 49.3   & 43.7   & 49.5 & 48.2 & 51.4 & 46.8 & 80.0 & 77.3 \\
				\rowcolor{customred} VAST \citep{Chen2023VASTAV}   & T-VAS & 50.9 & 47.9   & - & - & - & - & 82.1 & 78.7 \\
				\rowcolor{customred} VAST \citep{Chen2023VASTAV}   & T-VASD & 51.2 & 48.3   & - & - & - & - & 82.4 & 79.2 \\
				{GRAM} \citep{cicchetti2025gramian} &T-VAS& 54.2 &51.6&-&-&-&-&83.2&81.9\\
				{GRAM} \citep{cicchetti2025gramian} &T-VASD& 54.9 &51.3&-&-&-&-&83.8&82.1\\
				{Triangle} \citep{cicchetti2025triangle} &T-VA& 54.8 &51.9&54.6&52.2&59.6&54.4&84.0&80.1\\
				PMRL \citep{liu2026principled} &T-VAS &54.6&51.4&50.6&48.4&56.0&49.6&80.5&75.2\\
				\midrule
				\rowcolor{customgreen} BaryBind   & T-V &  54.2  &  52.0  & 54.8      & 52.9  &  59.3    & 52.8 & 82.9 & 80.4\\
				\rowcolor{customgreen}	BaryBind   & T-VA &  56.1    & 53.6    & \textbf{56.3}    & \textbf{54.0}  & \textbf{60.6}     & \textbf{57.2}  & 84.7   & 81.8  \\
				\rowcolor{customgreen}	BaryBind         & T-VAS  & 57.1 & 53.8 &  -      & -  &  -    & - &  85.5 & 84.3   \\
				\rowcolor{customgreen}	BaryBind         & T-VASD  & \textbf{57.8} & \textbf{54.6} &  -      & -  &  -    & - &  \textbf{86.1} & \textbf{85.4}   \\
				
				\midrule
				\textcolor{gray}{InternVideo2-6B \citep{Wang2024InternVideo2SV}}  & 	\textcolor{gray}{T-VA} &  	\textcolor{gray}{54.9} & 	\textcolor{gray}{49.1}    & 	\textcolor{gray}{55.7}      & 	\textcolor{gray}{51.6}     &  	\textcolor{gray}{61.2}    & \textcolor{gray}{52.8} & \textcolor{gray}{82.7}&\textcolor{gray}{76.4} \\	
				\bottomrule
			\end{tabular}
		}
		\centering
	\end{table*}
	
	\begin{table*}[!h]
		\centering
		\begin{minipage}{0.48\textwidth}
			\centering
			\caption{Text-to-audio (T2A) and audio-to-text (A2T) R@1 retrieval results on AudioCaps \textbf{w/ and w/o audio during training}.}
			\resizebox{0.91\linewidth}{!}{%
				\label{trainingmiss}
				\begin{tabular}{lccc}
					\toprule
					Method	&	Training modality & T2A  & A2T  \\
					\midrule
					\rowcolor{customred}	VAST		&  & 32.1 & 26.1 \\
					GRAM	& & 33.2 & 27.4 \\
					Triangle & T-VA & 32.2 &28.1\\
					\rowcolor{customgreen}	BaryBind	&	 & \textbf{35.7} & \textbf{32.5} \\
					\rowcolor{customgreen}	\textit{vs. baseline} (VAST)& &\textgr{\small\bf +3.6}&\textgr{\small\bf +6.4}\\
					\midrule
					\rowcolor{customred}	VAST	&	  & 10.4 & 6.7 \\
					GRAM	&	& 12.8 & 7.3\\
					Triangle &  & 11.6 & 7.1\\
					SMIL & T-V & 13.3&9.5\\
					ShaSpec & & 16.6&10.8\\
					\rowcolor{customgreen}	BaryBind	&	  & \textbf{21.2} & \textbf{14.5} \\
					\rowcolor{customgreen}	\textit{vs. baseline} (VAST) & & \textgr{\small\bf +10.8}& \textgr{\small\bf + 7.8}\\
					\bottomrule
				\end{tabular}%
			}
		\end{minipage}%
		\hfill
		\begin{minipage}{0.48\textwidth}
			\centering
			\caption{Multimodal event classification on VGGSound5K \textbf{w/ and w/o video during inference} when training with audio and video.}
			\resizebox{0.92\linewidth}{!}{%
				\label{testmiss}
				\begin{tabular}{lccc}
					\toprule
					Method	&Inference modality & Acc@1 & Acc@5 \\
					\midrule
					\rowcolor{customyellow}		VAST &  & 48.1 & 79.6 \\
					GRAM &  & 42.3 & 76.4 \\
					Triangle & A+V& 44.8 & 80.0\\
					\rowcolor{customblue}	BaryBind &  & \textbf{55.6} & \textbf{83.4} \\
					\rowcolor{customblue}	\textit{vs. baseline} (VAST) & &\textbl{\small\bf +7.5}&\textbl{\small\bf +3.8}\\
					\midrule
					\rowcolor{customyellow}	VAST &  & 40.8 & 71.6 \\
					GRAM &  & 38.5 & 70.1 \\
					Triangle & & 40.1 & 70.9\\
					SMIL & A & 43.5&70.2\\
					ShaSpec & & 46.2&74.5\\
					\rowcolor{customblue}		BaryBind &  & \textbf{49.4} &\textbf{78.3} \\
					\rowcolor{customblue}	\textit{vs. baseline} (VAST) & &\textbl{\small\bf +8.6}&\textbl{\small\bf +6.7}\\
					\bottomrule
				\end{tabular}%
			}
		\end{minipage}
	\end{table*}
	
	\subsubsection{Cross-modal retrieval} We evaluate BaryBind on zero-shot text-to-video (T2V) and video-to-text (V2T) retrieval across four benchmarks. As shown in Table~\ref{table:videoText}, BaryBind consistently achieves state-of-the-art results under all modality configurations. For instance, under the T-VA setting, it achieves 56.1 and 54.0 R@1 on DiDeMo (T2V/V2T), outperforming all prior methods. With T-VAS, BaryBind further sets new records on MSR-VTT and VATEX, demonstrating strong retrieval performance across both directions.	In addition to these overall gains, BaryBind effectively reduces the T2V/V2T performance gap,  such as 2.8 on MSR-VTT under the T-VA setting vs 5.6 of the VAST baseline, which highlights improved balanced alignment. This suggests that BaryBind learns more balanced, unified representations, which can be attributed to the WB optimization and volumetric alignment process for establishing a balanced unified semantic space.
	
	\subsection{Robustness to Missing Modalities}
	
	Barybind is designed to capture balanced representations from arbitrary subsets of multimodal inputs, enabling the aligned representations of available modalities to serve as proxy features when others are missing. We evaluate BaryBind under two settings: (1) text-to-audio (T2A) and audio-to-text (A2T) retrieval on AudioCaps \citep{kim2019audiocaps} with missing audio during training, and (2) multimodal event classification with missing video during inference on VGGSound, where the model is trained using both video and audio. In both cases, the barycenter is optimized using only the accessible modalities. Results in Table \ref{trainingmiss} and \ref{testmiss} show that BaryBind preserves robust and graceful degradation under missing-modality conditions as compared to traditional anchor-based methods VAST \citep{Chen2023VASTAV}, GRAM \cite{cicchetti2025gramian}, and missing modality-tailored methods SIML \cite{ma2021smil} and ShaSpec \cite{wang2023multi}, which demonstrate the effectiveness of our barycenter modeling for the shared semantics.

	\subsection{Ablation Studies}
	
	\begin{figure*}[!t]
		\centering
		\includegraphics[width=0.96\linewidth]{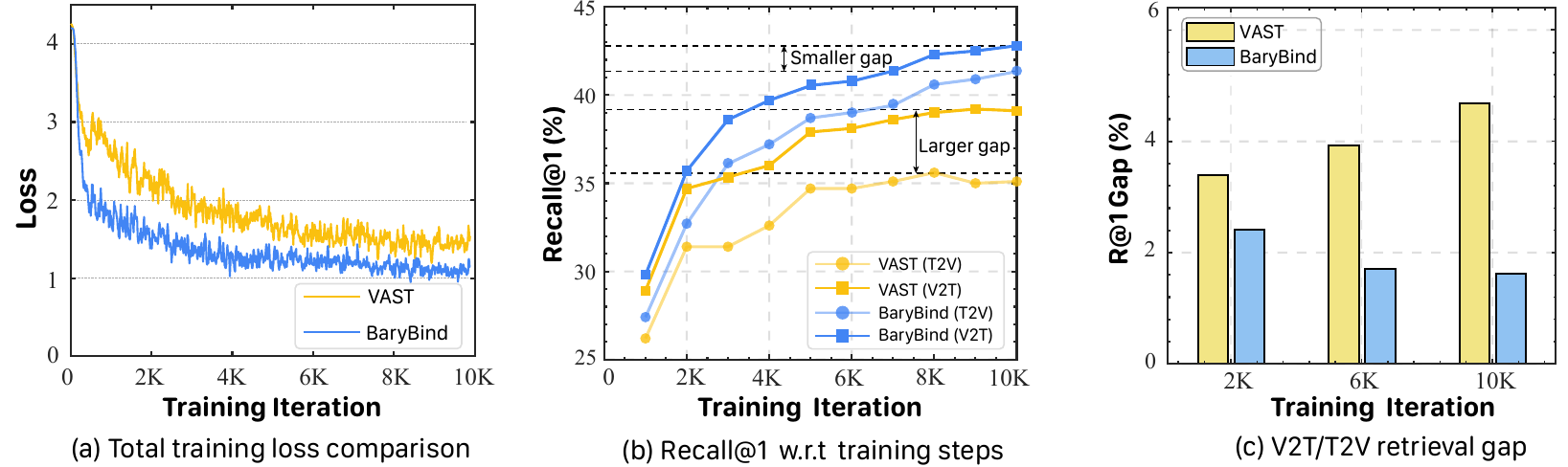}
		\caption{
			\textbf{Training dynamics comparison with cosine-based VAST:}
			(a) Training loss comparison between BaryBind and VAST.
			(b) Recall@1 (R@1) evolution on MSR-VTT under T2V and V2T retrieval tasks. 
			BaryBind achieves higher retrieval accuracy and better modal balance (smaller T2V/V2T R@1 gap). (c) V2T/T2V retrieval gaps throughout training.
		}
		\label{fig:loss_curve}
	\end{figure*}
	
	\begin{table*}[!h]
		\centering
		\caption{\textbf{Ablation study on loss functions.} We report the zero-shot top-1 classification accuracy (Acc@1) on VGGSound5K and top-1 retrieval recall at 1 score (R@1) on MSR-VTT. The performance gains over the \sethlcolor{gray!5}\hl{baseline} are highlighted accordingly. }
		\setlength{\tabcolsep}{13pt}
		\resizebox{\linewidth}{!}{	\begin{tabular}{cccclllll}
				\toprule
				\multicolumn{4}{c}{\multirow{2}{*}{Loss function components}} &\multicolumn{3}{c}{Classification}& \multicolumn{2}{c}{Retrieval}  \\
				\cmidrule(lr){5-7} \cmidrule(lr){8-9}
				&& &  &\multicolumn{3}{c}{VGGSound}& \multicolumn{2}{c}{MSR-VTT}  \\
				\midrule
				
				TV+TA CL & MWB& BVC& DAM&  A & V&V+A & T2V & V2T \\
				\midrule
				\bandzero \cmark & \xmark & \xmark & \xmark  &38.1&42.8& 44.5&46.8 & 40.1  \\
				\bandzero \cmark & \xmark & \xmark & \cmark  &40.3&43.6&46.3&49.3 & 43.7  \\
				\midrule
				\cmark &  \cellcolor{alicered}\cmark & \xmark & \xmark
				&43.6 \textbf{\textrd{\tiny\!(+5.5)}}&45.6&48.6 &49.1 & 47.3 \textbf{\textrd{\tiny\!(+7.2)}} \\
				\cmark & \cellcolor{alicered} \cmark & \xmark & \cmark &44.3 \textbf{\textrd{\tiny\!(+4.0)}}&46.8& 49.8&50.6 & 48.8 \textbf{\textrd{\tiny\!(+5.1)}}  \\
				\midrule
				\xmark &  \xmark & \cellcolor{aliceblue}\cmark & \xmark &42.2 &44.8&47.2  &50.6 & 46.4  \\
				\xmark &  \xmark & \cellcolor{aliceblue}\cmark & \cmark &42.9 &47.1&50.1  &51.5& 46.8  \\
				\midrule
				\xmark &\cellcolor{alicegreen}\cmark &\cellcolor{alicegreen}\cmark & \xmark & 45.2 \textbf{\textgr{\tiny\!(+7.1)}} &47.8 \textbf{\textgr{\tiny\!(+5.0)}}&54.7 \textbf{\textgr{\tiny\!(+10.2)}}&54.8 \textbf{\textgr{\tiny\!(+8.0)}} &51.1 \textbf{\textgr{\tiny\!(+11.0)}} \\
				
				\xmark &  \cellcolor{alicegreen}\cmark &\cellcolor{alicegreen}\cmark & \cmark &45.7 \textbf{\textgr{\tiny\!(+5.4)}}&48.3 \textbf{\textgr{\tiny\!(+4.7)}}&55.6 \textbf{\textgr{\tiny\!(+9.3)}}&56.1 \textbf{\textgr{\tiny\!(+6.8)}} & 53.6 \textbf{\textgr{\tiny\!(+9.9)}}  \\
				\bottomrule
		\end{tabular}}
		\label{tab:loss_ablation}
		\vspace{-0.3cm}
	\end{table*}
	\textbf{Training dynamics comparison.}
	We compare the training from scratch behaviors of BaryBind and the baseline VAST to assess the impact of our binding strategy and loss design. As shown in Fig.~\ref{fig:loss_curve}, BaryBind exhibits consistently superior training dynamics. In (a), the total loss drops faster and stabilizes earlier, suggesting that our MWB and BVC losses facilitate more efficient optimization. In (b), BaryBind consistently achieves higher R@1 accuracy for both T2V and V2T retrieval, while maintaining a smaller performance gap between the two directions, suggesting more symmetric multimodal representations. (c) demonstrates that BaryBind rapidly reduces the gap between V2T and T2V performance, converging toward a balanced retrieval behavior. This improvement stems from the joint effect of MWB and the BVC loss, which together promote consistent alignment of embeddings toward a unified barycenter. Such training behavior reflects our barycentric optimization converges stably, and also demonstrates the effectiveness of the barycenter-based alignment in constructing a balanced representation space.
	
	\textbf{Impact of loss functions.}
	To evaluate the contribution of each loss component, we conduct an ablation study on VGGSound~\citep{chen2020vggsound} for audio/video classification and on MSR-VTT~\citep{xu2016msr} for text–video retrieval under a tri-modal setting (text, video, audio). The baseline adopts pairwise cosine contrastive losses (TV+TA CL). As summarized in Table~\ref{tab:loss_ablation}, introducing the multimodal Wasserstein barycenter (MWB) loss consistently improves performance across tasks. For example, boosting audio-only classification from 40.3\% to 44.3\% and V2T retrieval from 43.7\% to 48.8\%. In this sense, MWB constructs a barycentric semantic anchor to filter out modality-specific biases from the original anchor (e.g., text). The barycenter volume contrastive (BVC) loss further enhances global consistency among modalities, enabling full-modality gains such as 55.6\%  in audio classification and 56.1/53.6 R@1 for T2V/V2T retrieval. Overall, the combination of MWB and BVC promotes coordinated alignment to the WB, yielding a more balanced representation space.

		\begin{wraptable}[8]{r}{0.45\textwidth}
		\vspace{-0.15in}
		\centering
		\caption{Ablation on batch size.}
		\label{tab:batch_size}
		\small
		\begin{tabular}{c|c|c}
			\toprule
			Batch Size & T2V/V2T R@1 & A+C Acc@1\\
			\midrule
			16 & 55.2/52.6 & 54.5\\
			32 & 55.8/53.4 & 55.1\\
			64 & 55.9/53.2 & 55.4\\
			128 & 56.1/53.5 & 55.8\\
			256 & 56.1/53.6 & 55.6\\
			\bottomrule
		\end{tabular}
		\vspace{-0.1in}
	\end{wraptable}
	
	\textbf{Effect of batch size.} Since BaryBind performs distribution-level alignment through Wasserstein barycentric optimization, we investigate the effect of batch size on the learned barycentric space. The training is conducted on two A100 GPUs without gradient accumulation. We evaluate under the T-VA setting on MSR-VTT retrieval and VGGSound classification. As shown in Table~\ref{tab:batch_size}, the performance remains stable across different batch sizes. Although a very small batch size (16) introduces a slight performance decrease due to noisier empirical distribution estimation, the overall gap is marginal. When the batch size is larger than 32, the performance becomes highly consistent, demonstrating the robustness of the learned barycentric representation under moderate batch-size variations.

	\section{Conclusion}
	In this work, we tackle the multimodal representation learning problem and propose BaryBind, a novel approach that aligns multiple modalities to a multimodal Wasserstein barycenter through a volumetric alignment objective. We demonstrate that this barycenter‑based alignment strategy is scalable and effectively alleviates modality imbalance in complex scenarios involving more than three modalities. Moreover, BaryBind exhibits robustness in missing‑modality settings. Overall, this work advances multimodal learning by introducing a barycenter-aligned joint representation that enables more balanced and effective multimodal understanding. In the future, we plan to further incorporate reconstruction into this framework to enable joint understanding and generation.
	
	%\section*{Acknowledgements}
	%
	%\textbf{Do not} include acknowledgements in the initial version of the paper
	%submitted for blind review.
	%
	%If a paper is accepted, the final camera-ready version can (and usually should)
	%include acknowledgements.  Such acknowledgements should be placed at the end of
	%the section, in an unnumbered section that does not count towards the paper
	%page limit. Typically, this will include thanks to reviewers who gave useful
	%comments, to colleagues who contributed to the ideas, and to funding agencies
	%and corporate sponsors that provided financial support.

	\nocite{langley00}

	\begin{ack}
		This work was supported in part by National Key R\&D Program under Grant 2021YFA1003000, in part by NSFC under Grant 125B2028, 12125104, 12426313,   12501709, and 12401671, in part by the National Postdoctoral Program for Innovative Talents under Grant BX20240276, in part by China Postdoctoral	Science Foundation under Grant 2025M773058, and in part by the Fundamental Research Funds for the Central Universities, China under Grant xzy022025047. 
	\end{ack}

	\bibliography{ref}
	\bibliographystyle{plainnat}
	%%%%%%%%%%%%%%%%%%%%%%%%%%%%%%%%%%%%%%%%%%%%%%%%%%%%%%%%%%%%

	\appendix
	\onecolumn
	\section{Theoretical Results}\label{proof}
	\subsection{Proof of theorem \ref{prop1}}
	\label{pf1}
	\begin{proof}
		Beyond the WB problem, we deduce based on the duality of with general costs $C_k(x,y)$. In this sense, we write the dual reformulation of the multimodal OT barycenter problem (\ref{MWB}):
		\begin{align}
			\label{dual-bary}
			\mathcal{L}^{*}=\inf_{\substack{\quad\\\mathbb{Q}\in\mathcal{P}(\mathcal{M}_B)}}\sup_{f_{0},\ldots,f_k\in\mathcal C(\mathcal M_B)}\underbrace{\sum_{k=1}^{K}\lambda_{k}\left\{\int_{\mathcal{M}_{k}}f_{k}^{C_{k}}(\bm m_k)\mathrm{d}\mathbb{P}_{k}(\bm m_k)+\int_{\mathcal{M}_B}f_{k}(\bm b)\mathrm{d}\mathbb{Q}(\bm b)\right\}}_{\triangleq\widetilde{\mathcal{F}}\left(\mathbb{Q},f_{1:K}\right)}.
		\end{align}
		where \begin{align}
			f_k^{C_k}(\bm m_k)=\inf_{b\in\mathcal M_B}[C_k(\bm m_k,\bm b)-f_k(\bm b)].
			\label{c-trans}
		\end{align}
		We denote the expression under the $\inf$ and $\inf\sup$ in (\ref{dual-bary}) as functionals $\mathcal L: \mathcal C(\mathcal M_B)^K$ and $\widetilde{\mathcal{F}}:\mathcal P(\mathcal M_B)\times\mathcal C(\mathcal M_B)^K$, respectively. For simplicity, we also introduce the following notations
		\begin{align}
			\label{notation}
			\bar f\triangleq\sum_{k=1}^K\lambda_kf_k~~~~\text{and}~~~~~M\triangleq\inf_{\bm b\in\mathcal M_B}\bar f(\bm b)=\inf_{\mathbb Q\in\mathcal P(\mathcal M_B)}\int_{\mathcal M_B}\bar f(\bm b)d\mathbb Q(\bm b),
		\end{align}
		where the equality follows from two fundamental observations:
		(a) $ M \leq \int \bar{f}(\bm b) \, d\mathbb{Q}(\bm b) \, \text{ for any } \mathbb{Q} \in \mathcal{P}(\mathcal{M}_B)$, and (b)  $\bar{f}(\bm b) = \int \bar{f}(\bm b') \, d\delta_{\bm b}(\bm b')$
		where $\delta_{\bm b}$ represents the Dirac mass at $ \bm b \in \mathcal{M}_B.$
		
		\quad\\
		\noindent Firstly, due to the compactness of $\mathcal M_B$, the space $\mathcal P(\mathcal M_B)$ is compact with respect to the weak topology. For fixed potentials $f_{1:K}\in\mathcal P(\mathcal M_B)^{K+1}$ we have that $\widetilde{\mathcal{F}}(\mathbb \cdot,f_{1:K})$ is linear, convex and continuous. Secondly, for a fixed $\mathbb Q$, the functional $\widetilde{\mathcal{F}}(\mathbb Q,\cdot)$ is a concave due to the concavity of $C$-transform. These properties enable the application of Sion's minimax theorem (\cite{sion1958general}, Theorem 3.4), which allows the interchange of the $\sup$ and $\inf$ in (\ref{dual-bary}). Thus with (\ref{notation}) we obtain
		\begin{align}
			\label{interchange}
			\mathcal{L}^*&=\sup_{f_{0},\ldots,f_{K}\in\mathcal{C}(\mathcal{M}_B)}\underbrace{\inf_{\mathbb{Q}\in\mathcal{P}(\mathcal{M}_B)}\sum_{k=1}^{K}\lambda_{k}\left\{\int_{\mathcal{M}_{k}}f_{k}^{C_{k}}(\bm m_k)\mathrm{d}\mathbb{P}_{k}(\bm m_k)+\int_{\mathcal{M}_B}f_{k}(\bm b)\mathrm{d}\mathbb{Q}(\bm b)\right\}}_{\triangleq\mathcal L(f_{1:K})}\nonumber\\&=\sup_{f_{0},\ldots,f_{K}\in\mathcal{C}(\mathcal{M}_B)}\left\{\sum_{k=1}^{K}\lambda_{k}\int_{\mathcal{M}_{k}}f_{k}^{C_{k}}(\bm m_k)\mathrm{d}\mathbb{P}_{k}(\bm m_k)+\inf_{\mathbb{Q}\in\mathcal{P}(\mathcal{M}_B)}\int_{\mathcal{M}_B}\bar{f}(\bm b)\mathrm{d}\mathbb{Q}(\bm b)\right\}\nonumber\\&=\sup_{f_0,...,f_K\in\mathcal{C}(\mathcal{M}_B)}\underbrace{\left\{\sum_{k=1}^K\lambda_k\int_{\mathcal{M}_k}f_k^{C_k}(\bm m_k)\mathrm{d}\mathbb{P}_k(\bm m_k)+\inf_{\bm b\in\mathcal{M}_B}\bar{f}(\bm b)\right\}}_{\triangleq\widetilde{\mathcal{L}}(f_{1:K})}.\end{align}
		Now we show that the $\sup$ in (\ref{interchange}) can be restricted to potentials $\tilde f_{1:K}$ which satisfy the congruence condition $\sum_{k=1}^K\lambda_k\tilde f_k=1$. It is enough to show that for every tuple $f_{1:K}$ there exists a congruent tuple $\tilde f_{1:K}\in\mathcal C(\mathcal M_B)^K$ such that $ \widetilde{\mathcal{L}}(\tilde f_{1:K})\geq \widetilde{\mathcal{L}}(f_{1:K})$.
		
		\noindent
		For this, we consider the congruent potentials given any tuple $f_{1:K}$
		\begin{align}
			(\tilde f_0,\cdots,\tilde f_K)=\left(f_0,\cdots,f_{K-1},f_K-\frac{\bar f}{\lambda_K}\right).
		\end{align}
		Since $\tilde M\triangleq\inf_{\bm b\in\mathcal M_B}\sum_{k=1}^K\lambda_k\tilde f_k=1$, we obtain
		\begin{align}
			\widetilde{\mathcal{L}}(\tilde f_{1:K})- \widetilde{\mathcal{L}}(f_{1:K})&=\lambda_K\int_{\mathcal M_k}\left(\tilde f_K^{C_K}(\bm m_K)-f_K^{C_K}(\bm m_K)\right)d\mathbb P_K(\bm m_K)-M\nonumber\\
			&=\lambda_K\int_{\mathcal M_k}\left[ \left(f_K-\frac{\bar f}{\lambda_K}\right)^{C_K}(\bm m_K)-f_K^{C_K}(\bm m_K)\right]d\mathbb P_K(\bm m_K)-M\nonumber\\
			&\geq\lambda_K\int_{\mathcal M_k}\left[ \left(f_K-\frac{M}{\lambda_K}\right)^{C_K}(\bm m_K)-f_K^{C_K}(\bm m_K)\right]d\mathbb P_K(\bm m_K)-M\nonumber\\&=\lambda_K\int_{\mathcal M_k}\frac{M}{\lambda_K}d\mathbb P_K(\bm m_K)-M=0,
		\end{align}
		where the first inequality arises from the monotonicity of the $C$-transform, along with the fact $\bar f_K=f_K-\frac{\bar f}{\lambda_K}\leq f_K-\frac{M}{\lambda_k}$. The last equality follows from the definition of $C$-transform.
		
		\noindent Finally, since $\widetilde{\mathcal L}(f_{1:K})={\mathcal L}(f_{1:K})$ for congruent potentials $f_{1:K}$, with (\ref{c-trans}) we obtain
		\begin{align}
			\mathcal L^*=\sup\limits_{\sum_k\!\lambda_k f_k = 0}{\mathcal L}(f_{1:K})&=\sup\limits_{\sum_k\!\lambda_k f_k = 0}\sum_{k=1}^K\lambda_k\int_{\mathcal{M}_k}f_k^{C_k}(\bm m_k)\mathrm{d}\mathbb{P}_k(\bm m_k),\nonumber\\
			&=\sup\limits_{\sum_k\!\lambda_k f_k = 0}\sum_{k=1}^K\lambda_k\inf_{\substack{\\[0.1ex]\bm b\in\mathcal M_B}}\int_{\mathcal{M}_k}\big[C_k(\bm m_k-\bm b)-f_k(\bm b)\big]\mathrm{d}\mathbb{P}_k(\bm m_k).
		\end{align}
		In practice, we replace each integral with an empirical expectation over the distribution $\mathbb{P}_k$, and adopt $C_k(\bm{m}_k - \bm{b}) = \|\bm{m}_k - \bm{b}\|$, which corresponds to the Wasserstein distance between the modality feature and the barycenter. We exchange the summation and the infimum since $\bm{b}$ is shared across all terms and the objective is linear. We then sample the barycenter $\bm{b}$ from a distribution $\mathbb{Q} \in \mathcal{P}(\mathcal{M}_B)$, and express the integral as an expectation over both $\bm{m}_k \sim \mathbb{P}_k$ and $\bm{b} \sim \mathbb{Q}$:
		\begin{align}
			\mathcal L^*=\sup\limits_{\sum_k\!\lambda_k f_k = 0}~\inf_{\substack{\\[0.1ex]\mathbb Q\in\mathcal P(\mathcal M_B)}}\sum_{k=1}^K\lambda_k\mathop{\mathbb E}_{\substack{\bm m_k\sim\mathbb P_k\\\bm b\sim\mathbb Q}}\big[\|\bm m_k-\bm b\|-f_k(\bm b)\big],
		\end{align}
		which completes the proof.
	\end{proof}
	
	\subsection{Proof of Theorem \ref{error}}
	\label{pf2}
	
	\begin{proof}
		Let $\mathbb{P}_{1:K}$ denote the joint distribution over the multimodal tuples $(\bm m_0, \dots, \bm m_K)=:\bm m_{1:K}$, where $\bm m_k \in \mathcal{M}_k$. The barycenter map is a function $T: \mathcal{M}_0 \to \mathcal{M}_B$ that acts on the anchor modality. For a set of potential functions $f_{1:K}$ and a map $T$, we rewrite the total cost functional $\mathcal{F}$ and its corresponding minimal cost functional $\mathcal{L}$ as follows:
		\begin{align}
			\mathcal{F}(f_{1:K}, T) &\triangleq \mathbb{E}_{\bm m_{1:K} \sim \mathbb{P}_{1:K}} \left[ \sum_{k=1}^{K} \lambda_k \left( C_k(\bm m_k, T(\bm m_0)) - f_k(T(\bm m_0)) \right) \right], \label{Ff} \\
			\mathcal{L}(f_{1:K}) &\triangleq \inf_{T: \mathcal{M}_0 \to \mathcal{M}_B} \mathcal{F}(f_{1:K}, T). \label{Lf}
		\end{align}
		The minimizer of functional $\mathcal F$ is denoted as:
		\begin{align}
			T^f\in\mathop{\arg\inf}_{T:\mathcal M\rightarrow \mathcal M_B}\mathcal F(\widehat f_{1:K},T).
		\end{align}
		Given $T^f:\mathcal M_0\rightarrow \mathcal M_B$, the functional $\mathcal L$ in  (\ref{Lf}) can be written as:
		\begin{align}
			\mathcal L(\widehat f_{1:K})=\mathcal{F}(\widehat f_{0:k},T^f).
			\label{LTf}
		\end{align}
		We can observe that the first gap $\mathcal E_1$  is the difference between (\ref{Ff}) and (\ref{LTf}):
		\begin{align}
			\mathcal{E}_1(\widehat{f}_{1:K}, \widehat{T}) = \mathcal{F}(\widehat{f}_{1:K}, \widehat{T}) - \mathcal{L}(\widehat{f}_{1:K}) = \mathcal{F}(\widehat{f}_{1:K}, \widehat{T}) - \mathcal{F}(\widehat{f}_{1:K}, T^f).
			\label{e1}
		\end{align}
		Before looking into the second gap $\mathcal E_2$, we recall the optimal value $\mathcal L^*$ of the OT barycenter problem and express the OT cost with Monge's formulation:
		\begin{align}
			\label{eq:L_star_definition}
			\mathcal{L}^* \triangleq \sum_{k=1}^{K} \lambda_k \text{OT}_{C_k}(\mathbb{P}_k, \mathbb{Q}^*).
		\end{align}
		where $\mathbb Q^*$ is the true barycenter distribution. By introducing Monge's OT formulation with the true OT map $T^*$ that satisfies $T^*_{\#}\mathbb{P}_0 = \mathbb{Q}^*$, the expression can be rewritten as:
		\begin{align}
			\label{eq:L_star_monge}
			\mathcal{L}^* = \sum_{k=1}^{K} \lambda_k \int_{\mathcal{M}_k} C_k(\bm m_k, T^*(\bm m_0)) d\mathbb{P}_k(\bm m_k).
		\end{align}
		\noindent Due to the congruence condition on the potentials $\widehat{f}_{1:K}$ and the property $T^*_{\#}\mathbb{P}_0 = \mathbb{Q}^*$ for all $k$, we have:
		\begin{align}
			\sum_{k=1}^{K} \lambda_k \mathbb{E}_{\bm m_k \sim \mathbb{P}_k} [\widehat{f}_k(T^*(\bm m_0))] = \mathbb{E}_{\bm b \sim \mathbb{Q}^*} \left[ \sum_{k=1}^{K} \lambda_k \widehat{f}_k(\bm b) \right]  = 0.
			\label{eq:congruence_zero}
		\end{align}
		This allows us to reformulate the optimal value $\mathcal{L}^*$. Using the definition of $\mathcal{F}$ in (\ref{Ff}), we find:
		\begin{align}
			\mathcal{L}^* &= \sum_{k=1}^{K} \lambda_k \mathbb{E}_{\bm m_k \sim \mathbb{P}_k}[C_k(\bm m_0, T^*(\bm m_k))] - \underbrace{\mathbb{E}_{\bm m_{1:K} \sim \mathbb{P}_{1:K}} \left[ \sum_{k=1}^{K} \lambda_k \widehat{f}_k(T^*(\bm m_0)) \right]}_{=0 \text{ from } (\ref{eq:congruence_zero})} \nonumber \\
			&= \mathbb{E}_{\bm m_{1:K} \sim \mathbb{P}_{1:K}} \left[ \sum_{k=1}^{K} \lambda_k \left( C_k(\bm m_k, T^*(\bm m_0)) - \widehat{f}_k(T^*(\bm m_0)) \right) \right] \nonumber = \mathcal{F}(\widehat{f}_{1:K}, T^*).
		\end{align}
		With (\ref{LTf}) we derive the second gap $\mathcal E_2$ can be written as
		\begin{align}
			\mathcal E_2=\mathcal L^*-\mathcal L(\widehat f_{1:K})= \mathcal{F}(\widehat{f}_{1:K}, T^*) - \mathcal{F}(\widehat{f}_{1:K}, T^f).
			\label{e2}
		\end{align}
		We introduce the function $g_k(\bm m_k, \bm b) \triangleq C_k(\bm m_k, \bm b) - \widehat{f}_k(\bm b)$, which is assumed to be $\beta$-strongly convex with respect to $\bm b$. Using this, we can rewrite our total cost functional $\mathcal{F}$ from (\ref{Ff}) as:
		\begin{align}
			\mathcal{F}(f_{1:K}, T) = \mathbb{E}_{\bm m_{1:K} \sim \mathbb{P}_{1:K}} \left[ \sum_{k=1}^{K} \lambda_k g_k(\bm m_k, T(\bm m_0)) \right].
		\end{align}
		As a result of convexity, it follows that a necessary condition for $T^f$ to minimize $\mathcal{F}(f_{1:K},T)$ is the vanishing of its first variation, yielding
		\begin{align}
			\mathbb{E}_{\bm m_{1:K} \sim \mathbb{P}_{1:K}} \left[ \sum_{k=1}^{K} \lambda_k \nabla_{\bm b} g_k(\bm m_k, T^f(\bm m_0)) \right] = 0.
			\label{eq:grad_zero_new}
		\end{align}
		Now, we analyze the gap $\mathcal{E}_1$ by applying the $\beta$-strong convexity of $g_k(\bm m_k, \cdot)$:
		\begin{align}
			\mathcal{E}_1 &= \mathcal{F}(\widehat{f}_{1:K}, \widehat{T}) - \mathcal{F}(\widehat{f}_{1:K}, T^f) \nonumber \\
			&= \mathbb{E}_{\bm m_{1:K} \sim \mathbb{P}_{1:K}} \left[ \sum_{k=1}^{K} \lambda_k \left( g_k(\bm m_k, \widehat{T}(\bm m_0)) - g_k(\bm m_k, T^f(\bm m_0)) \right) \right] \nonumber \\
			&\geq \mathbb{E}_{\bm m_{1:K} \sim \mathbb{P}_{1:K}} \left[ \sum_{k=1}^{K} \lambda_k \left( \langle \nabla_{\bm b} g_k(\bm m_k, T^f(\bm m_0)), \widehat{T}(\bm m_0) - T^f(\bm m_0) \rangle + \frac{\beta}{2} \|\widehat{T}(\bm m_0) - T^f(\bm m_0)\|^2 \right) \right] \nonumber \\
			&= \mathbb{E}_{\bm m_{1:K} \sim \mathbb{P}_{1:K}} \left[ \left\langle \sum_{k=1}^{K} \lambda_k \nabla_{\bm b} g_k(\bm m_k, T^f(\bm m_0)), \widehat{T}(\bm m_0) - T^f(\bm m_0) \right\rangle \right] + \frac{\beta}{2} \mathbb{E}_{\bm m_0 \sim \mathbb{P}_0} \left[ \|\widehat{T}(\bm m_0) - T^f(\bm m_0)\|^2 \right] \nonumber \\
			&\overset{(\ref{eq:grad_zero_new})}{=} 0 + \frac{\beta}{2} \mathbb{E}_{\bm m_0 \sim \mathbb{P}_0} \left[ \|\widehat{T}(\bm m_0) - T^f(\bm m_0)\|^2 \right]. \label{in-e1}
		\end{align}
		For the second gap $\mathcal E_2$, we conduct the same analysis and obtain
		\begin{align}
			\mathcal E_2\geq  \frac{\beta}{2} \mathbb{E}_{\bm m_0 \sim \mathbb{P}_0} \left[ \|T^f(\bm m_0) - T^*(\bm m_0)\|^2\right].
			\label{in-e2}
		\end{align}
		Now we sum the inequalities for $\mathcal E_1$ (\ref{in-e1}) and $\mathcal E_2$ (\ref{in-e2}):
		\begin{align}
			\mathcal{E}_1 + \mathcal{E}_2 &\geq \frac{\beta}{2} \mathbb{E}_{\bm m_0 \sim \mathbb{P}_0} \left[ \|\widehat{T}(\bm m_0) - T^f(\bm m_0)\|^2 \right] + \frac{\beta}{2} \mathbb{E}_{\bm m_0 \sim \mathbb{P}_0} \left[ \|T^f(\bm m_0) - T^*(\bm m_0)\|^2 \right] \nonumber \\
			&= \frac{\beta}{2} \mathbb{E}_{\bm m_0 \sim \mathbb{P}_0} \left[ \|\widehat{T}(\bm m_0) - T^f(\bm m_0)\|^2 + \|T^f(\bm m_0) - T^*(\bm m_0)\|^2 \right] \nonumber \\
			&\geq \frac{\beta}{4} \mathbb{E}_{\bm m_0 \sim \mathbb{P}_0} \left[ \|\widehat{T}(\bm m_0) - T^*(\bm m_0)\|^2 \right] \geq \frac{\beta}{4} W^2(\widehat{T}_{\#}\mathbb{P}_0, T^*_{\#}\mathbb{P}_0) = \frac{\beta}{4} W^2(\widehat{T}_{\#}\mathbb{P}_0, \mathbb{Q}^*).
		\end{align}
	\end{proof}
	\subsection{Analysis on barycenter simplex volume}
	\label{Volanalysis}
	Given the matrix containing vectors spanning the barycenter simplex $$\mathbf R=(\bm b,\bm r_1,\cdots,\bm r_K),$$ we discuss the geometric meaning of barycenter simplex volumes for $n=2$ and $n=3$ modalities. \(\langle \cdot, \cdot \rangle\) denotes the cosine similarity between two normalized vectors, so that
	\begin{align}
		\label{unit}
		\langle \bm{b}, \bm{b} \rangle = \langle \bm{r}_1, \bm{r}_1 \rangle = \langle \bm{r}_2, \bm{r}_2 \rangle = 1.
	\end{align}
	We define the cosine of the angles between these vectors as
	\begin{align}
		\label{cos}
		\cos \theta := \langle \bm{b}, \bm{r}_1 \rangle, \quad
		\cos \beta := \langle \bm{b}, \bm{r}_2 \rangle, \quad
		\cos \gamma := \langle \bm{r}_1, \bm{r}_2 \rangle.
	\end{align}
	
	For the special case \(n=2\), the squared volume can be expressed as the determinant
	\[
	\mathrm{Vol}^2_{n=2} = \det(\mathbf{R}^\top \mathbf{R}) = 
	\begin{vmatrix}
		\langle \bm{b}, \bm{b} \rangle & \langle \bm{r}_1, \bm{b} \rangle  \\
		\langle \bm{b}, \bm{r}_1 \rangle & \langle \bm{r}_1, \bm{r}_1 \rangle 
	\end{vmatrix}=1-\cos^2\theta
	\]
	Therefore, the simplex volume in the bimodal case reduces to:
	\begin{align*}
		\text{Vol}_{n=2}=\sqrt{1-\cos^2\theta}=\sqrt{\sin^2\theta}=\sin\theta,
	\end{align*}
	which quantifies the spatial deviation between the WB embedding vector $\bm b$  and the vector $\bm b-\bm m_1$, reflecting how well the modality aligns with the overall barycenter direction. A smaller volume indicates stronger alignment, while a larger value implies greater directional discrepancy.
	
	For the special case \(n=3\), the squared volume can be expressed as the determinant
	\[
	\mathrm{Vol}^2_{n=3} = \det(\mathbf{R}^\top \mathbf{R}) = 
	\begin{vmatrix}
		\langle \bm{b}, \bm{b} \rangle & \langle \bm{r}_1, \bm{b} \rangle & \langle \bm{r}_2, \bm{b} \rangle \\
		\langle \bm{b}, \bm{r}_1 \rangle & \langle \bm{r}_1, \bm{r}_1 \rangle & \langle \bm{r}_2, \bm{r}_1 \rangle \\
		\langle \bm{b}, \bm{r}_2 \rangle & \langle \bm{r}_1, \bm{r}_2 \rangle & \langle \bm{r}_2, \bm{r}_2 \rangle
	\end{vmatrix}.
	\]
	Substituting (\ref{unit}) and (\ref{cos}) into the determinant, the expression simplifies to
	\begin{align*}
		\mathrm{Vol}^2_{n=3} &= 1 - \cos^2 \theta - \cos^2 \beta - \cos^2 \gamma + 2 \cos \theta \cos \beta \cos \gamma,\\&=\sin^2\theta+\sin^2\beta+\sin^2\gamma+2\cos\theta\cos\beta\cos\gamma-2.
	\end{align*}
	
	This form reveals the geometric interpretation more clearly: 1) The $\sin^2$ terms quantify the angular deviation of each vector pair. The $2\cos\theta\cos\beta\cos\gamma$ term reflects the angular coupling between all three directions. 2) $\gamma$ is the angle between the two modality gap vectors $\bm{r}_1$ and $\bm{r}_2$, and thus directly reflects the  structural inter-modality gaps between non-anchor modalities.
	
	\begin{wrapfigure}[16]{R}{0.36\textwidth}
		\centering
		\vspace{-0.3cm}
		\includegraphics[width=1.0\linewidth]{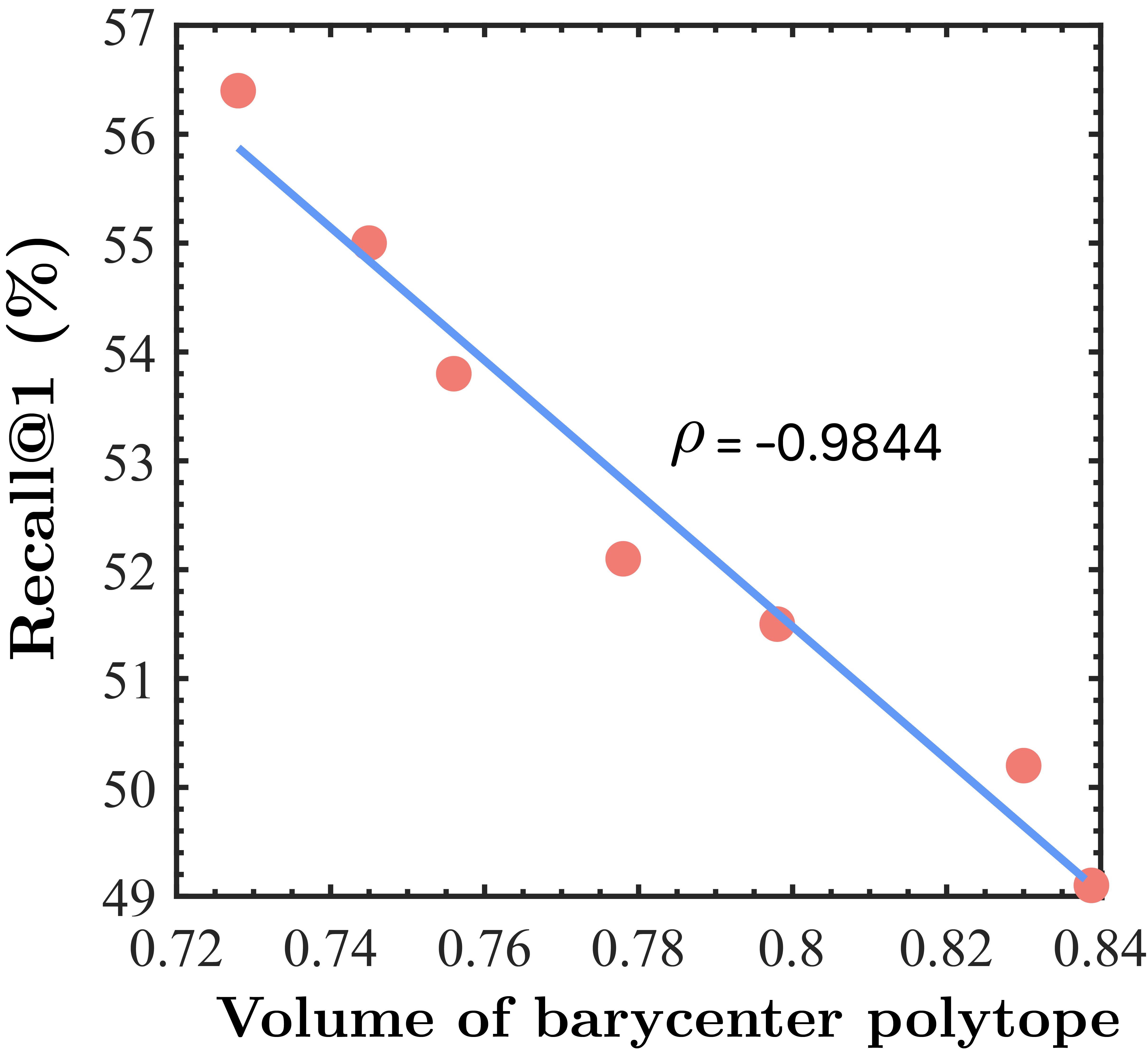}
		\vspace{-0.5cm}
		\caption{The simplex's volume is correlated ($\rho=-0.9844$) with the downstream retrieval performance.}
		\label{corr}
	\end{wrapfigure}
	The volume becomes small when each modality vector closely follows the barycenter direction (small angles $\alpha$ and $\beta$). In addition, if the gap vectors between modalities, $\bm{r}_1$ and $\bm{r}_2$, are nearly aligned (small inter-modality angle $\gamma$), the volume is further reduced.  In this case, all three angles are small, their sines are close to $0$, and their cosines are close to $1$. Conversely, the volume increases when the directions are mutually orthogonal, which maximizes the total dispersion among them. Fig.~\ref{corr} plots the R@1 values of zero-shot retrieval on MSR-VTT with respect to the volume, showing a clear negative correlation between volume and performance. 
	
	Unlike pairwise cosine similarity, which only captures alignment between two modalities, the barycenter simplex volume offers a rich higher-order understanding of multimodal structure. It not only measures the holistic geometric alignment of each modality toward the barycenter, but also preserves inter-modal interactions, yielding a more compact multimodal joint representation space. In this sense, the volume serves as a non-trivial metric for measuring $n$-modality alignment.
	
	\textbf{Geometric Intuition.}  Geometrically, traditional methods rely on asymmetric alignment, forcing non-anchor modalities (e.g., video) to conform to a modality-specific anchor (e.g., text). In contrast, BaryBind establishes a shared "geometric consensus" by optimizing a Wasserstein barycenter across all modalities. Instead of leaning towards a fixed specific modality, the multimodal embeddings of BaryBind converge toward the central manifold supported by the refined WB anchor. This balanced constraint shares the optimization burden and mitigates modality-specific bias, logically explaining the reduced T2V/V2T performance gap observed in our experiments.

	\section{Experimental Details and More Results}
	\label{exp-d}
	
	\subsection{Training algorithm}
	\label{algorithm}
	\begin{algorithm}[!h]
		\caption{Training Algorithm}  
		\textbf{Input}: Per-modality encoders $E_{\phi_{1:K}}$; Multimodal data set $\mathcal{S}_{1:K}$; NN-based WB map $T_\theta$; NN-based potentials $f_{\omega_{1:K}}$; Learning rate $\gamma$\\ 
		\textbf{Output}: Learned $T_{\theta}$,$E_{\phi_{1:K}}$,$f_{\omega_{1:K}}$
		\begin{algorithmic}[1]
			\label{algo}
			\WHILE{not converged}
			\STATE Sample multimodal batch data $\bm z_{1:K}\sim\mathcal{S}_{1:K}$
			\STATE Extract embeddings $\{\bm m_{k} = E_{\phi_k}(\bm z_{k}) \sim \mathbb{P}_k\}_{k=1}^K$
			\STATE Compute WB anchor $\bm b = T_\theta(\bm m_0)$
			\STATE Compute $\mathcal{L}_{\mathrm{MWB}}$ by Eqn. (\ref{Lbary})
			\STATE Compute $\mathcal{L}_{\mathrm{BVC}}, \mathcal{L}_{\mathrm{DAM}}, \mathcal{L}$ by Eqns. (\ref{eq:vcl}), (\ref{bdm}),  (\ref{tLoss})
			\STATE \# Optimize network blocks
			\STATE $\omega_k = \omega_k + 2\gamma \frac{\partial \mathcal{L}_{\mathrm{MWB}}}{\partial\omega_k}$, $k=1,\cdots, K$ 
			\STATE $\phi_k = \phi_k - \gamma \frac{\partial \mathcal{L}}{\partial \phi_k}$, $k=1,\cdots, K$
			\STATE $\theta = \theta - \gamma \frac{\partial \mathcal{L}}{\partial\theta}$

			%\STATE $\mathcal L^f_{\mathrm{MWB}}\leftarrow\frac{1}{|B|}\sum_{\bm z_k\in B}\lambda_kf_{\omega_k}(T_\theta(\bm z_k))$; 
			%\STATE	Update $\omega_{1:K}$ by using $\frac{\partial \mathcal L^f_{MWB}}{\partial \omega_{1:K}}$;
			%\FOR{$t=0,\cdots,n_T$}
			%\STATE Sample batches $B$ with $\bm z_k\sim\mathbb P_k,k\in\bar K$;
			%\STATE $\mathcal \displaystyle \mathcal L^T_{\mathrm{MWB}}\leftarrow\frac{1}{|B|}\bigg\{\sum_{\bm m_k\in B}\lambda_k\big[\|\bm m_k-T_\theta(\bm m_k)\|-f_{\omega_k}(T_\theta(\bm m_k))\big]\bigg\}$;
			%\STATE	Update $T_\theta$ by using $\frac{\partial \mathcal L^T_{\mathrm{MWB}}}{\partial\theta}$;
			%\ENDFOR
			\ENDWHILE
		\end{algorithmic} 
	\end{algorithm}
	
	\subsection{Experimental Details}
	\label{details}
	We perform both continued pretraining on top of VAST models and training from scratch of the BaryBind model using 2×A100 GPUs. For continued pretraining, We sampled 150,000 examples from the VAST-27M dataset \citep{Chen2023VASTAV} due to the unavailability of the full dataset. Each example consists of a video (comprising frames and an audio track) paired with a corresponding caption. We further evaluate the training-from-scratch models on MSR-VTT and ActivityNet in \S\ref{scratch}. The model backbone follows the VAST multimodal encoder, utilizing BERT-B, modality-specific audio and vision encoders, and other modality encoders as applicable, resulting in a total of approximately 1.3 billion parameters.  
	
	During the continued pretraining phase, we use a single frame and a single 10-second audio clip per sample; in the training-from-scratch setting, we use eight randomly sampled video frames and a single 10-second audio clip per video.
	
	For the continued pretraining, BaryBind is trained for 10k steps on the 150k-sample subset of VAST27M. Retrieval performance is evaluated every 100 steps on the MSR-VTT test set, and the checkpoint with the best performance is selected. We employ an initial learning rate of 1e-4 with a linear decay schedule and a batch size of 256. In the training-from-scratch experiments, all encoders are trained from scratch on the MSR-VTT training dataset for 4 epochs, using an initial learning rate of 1e-4 with a linear decay schedule and a batch size of 64.
	
	\begin{table}[!h]
		\centering
		\caption{Overview of multimodal benchmarks used for downstream evaluation.}
		\label{tab:datasets}
		\begin{adjustbox}{width=0.8\textwidth,center}
			\begin{tabular}{lcccccc}
				\toprule
				Benchmark       & Modalities                          & Train         & Val         & Test        & \# Frames (train) &\# Frames (test) \\
				\midrule
				DiDeMo        & Video + Text + Audio               & 8,394            & 1,065           & 1,003           & 8 & 32\\
				ActivityNet   & Video + Text + Audio               & 10,009         & —           & 4,917           & 8  & 32 \\
				MSR‑VTT       & Video + Text + Audio + Subtitle   & 9,000  & —   & 1,000           & 8 & 8 \\
				VATEX         & Video + Text + Audio + Subtitle   & 14,060         & —               & 431               & 8 & 16 \\
				VGGSound      & Audio + Video + Text              & —                & —               & ~5000  & - & 8 \\
				\bottomrule
		\end{tabular}\end{adjustbox}
	\end{table}
	
	Table~\ref{tab:datasets} summarizes the modality configuration, data statistics, and frame settings across benchmarks.
	
	\noindent\textbf{MSR-VTT}~\citep{xu2016msr} is a widely used benchmark for video-text retrieval. It contains 10,000 video clips, each paired with approximately 20 textual captions, totaling around 200,000 captions. We use 9,000 videos for training and 1,000 for testing, with 8 frames sampled per video.
	
	\noindent\textbf{DiDeMo}~\citep{anne2017localizing} consists of 10,000 long-form videos, each annotated with 4 temporally ordered paragraph descriptions. It is mainly used for moment localization retrieval. The official split includes 8,394/1,065/1,003 videos for training/validation/testing, and 12 frames are sampled per video.
	
	\noindent\textbf{ActivityNet}~\citep{caba2015activitynet} contains about 20,000 YouTube videos with a total duration of approximately 180 hours, annotated with multiple temporal sentence descriptions. We use the official training set (10,009 videos) for training and the validation set (4,917 videos) for downstream testing. The number of sampled frames per video is 8.
	
	\noindent\textbf{VATEX}~\citep{wang2019vatex} includes around 25,000 English videos, each annotated with 10 English captions. It is commonly used for four-modality text retrieval tasks. We adopt 14,060 videos for training and 431 for testing.
	
	\noindent\textbf{VGGSound5K}~\citep{chen2020vggsound} is a 5,000-video subset of VGGSound, containing diverse audio event categories (typically 310 classes) along with corresponding video frames and subtitles. Each video clip
	in the dataset has a duration of 10 seconds and is annotated with a single label corresponding to the predominant sound event occurring within the clip. The dataset covers a wide spectrum of audio events, including human actions, animal vocalizations, natural phenomena, and mechanical sounds.It is commonly used for audio-video multimodal classification tasks. 
	
	\subsection{Broader Multimodal Understanding and Generation}
	
Multimodal learning has recently expanded beyond conventional image--text understanding to complex temporal, embodied, and generative scenarios. In autonomous systems, StreamVLO integrates visual and LiDAR observations with spatio-temporal correlation modeling for robust streaming odometry~\cite{liu2026streamvlo}, while DriveVA jointly models future videos and driving actions to transfer video-generation priors to zero-shot planning~\cite{liu2026driveva}. UNIVERSE further unifies video and trajectory generation within a mask-modulated architecture~\cite{liu2026universe}, and OmniDrive coordinates language, geometry, and multi-view visual signals for controllable driving-world generation~\cite{meng2026omnidrive}. These developments highlight the growing need to organize heterogeneous sensory and semantic information within a coherent representation space. Reliable multimodal understanding additionally requires preserving and faithfully exploiting visual evidence. MemVR mitigates hallucinations by reinjecting visual information into the model's memory space when relevant evidence has faded~\cite{zoulook}. HoloV instead studies holistic visual-token retention, showing that maintaining globally distributed visual context is important for efficient multimodal reasoning~\cite{zou2026don}. OPPO further calibrates generation preferences according to the strength of visual evidence, improving visual grounding and reducing hallucinations~\cite{zou2026clearer}. Related challenges also arise in visual generation and evaluation: ARGUS aggregates multi-view identity evidence for subject-consistent video generation~\cite{meng2026argus}, MST-CLIPIQA decouples semantic coherence from perceptual distortions through multi-scale vision--language alignment~\cite{meng2026decoupling}, and KAE develops unified knowledge embedding for few-shot image generation~\cite{xu2025keep}. Collectively, these studies demonstrate that modern multimodal systems must reconcile heterogeneous modalities, granularities, and objectives while retaining both global semantics and modality-specific information. BaryBind addresses this general representation challenge from an optimal-transport perspective, using a multimodal Wasserstein barycenter as a shared anchor to bind multiple modalities without collapsing their complementary structures.
	
	\subsection{Finetuning Cross-modal Retrieval Results}
	We present the fine-tuning results in Table~\ref{table:videoText_finetuning}, which shows BaryBind maintains its advantage, outperforming prior models such as InternVideo, VideoCLIP-XL, and VAST. For instance, our T-VA setting achieves 60.3 R@1 on MSR-VTT and 72.1 on ActivityNet for T2V, surpassing VAST by +4.5 and +2.3 points, respectively. The performance further improves under T-VAS, where BaryBind reaches 65.7 R@1 on MSR-VTT, indicating that our binding strategy generalizes well with additional modalities and scales effectively with fine-tuning.
	\begin{table*}
		\caption{\textbf{Finetuning} text-to-video (T2V) and video-to-text (V2T) retrieval results in terms of Recall at 1 score (R@1). Results from our baseline \sethlcolor{customred}\hl{VAST} and  \sethlcolor{customgreen}\hl{BaryBind} are highlighted accordingly.}
		\label{table:videoText_finetuning}
		\vspace{-0.2cm}
		\setlength{\tabcolsep}{10pt}
		\renewcommand{\arraystretch}{1}
		\resizebox{\textwidth}{!}{
			\begin{tabular}{lc cc cc cc cc}
				\toprule
				\multirow{2}{*}{Methods}&\multirow{2}{*}{Modality} & \multicolumn{2}{c}{MSR-VTT} & \multicolumn{2}{c}{DiDeMo} & \multicolumn{2}{c}{ActivityNet}  & \multicolumn{2}{c}{VATEX}   \\ 
				\cmidrule(lr){3-4}
				\cmidrule(lr){5-6}
				\cmidrule(lr){7-8}
				\cmidrule(lr){9-10}
				&  & T2V    & V2T   & T2V     & V2T  & T2V    & V2T  & T2V    & V2T   \\ \midrule
				
				% UMT \citep{liu2022umt}                        & T-VA & 33.3   &    & 34.0     &   & 31.9   &   & -&-\\
				% OmniVL \citep{Wang2022OmniVLOF} & T-V & 34.6    &    & 33.3    &   & -   & -  & -&-\\
				
				CLIP4Clip \citep{Luo2021CLIP4ClipAE} & T-V & 45.6  & 45.9     & 43.0 & 43.6     & 40.3  & 41.6  & 63.0 & 78.3 \\
				% TVTSv2 \citep{Zeng2023TVTSv2LO} & T-V & 38.2       &       & 34.6       &      & -     & -  & -&-\\
				
				% VideoCoCa \citep{yan2022videococa}  & T-V & 34.3    & 64.7   & -         & -     & 34.5     & 33.0 & 20.3 &  \\
				% Norton \citep{lin2024norton}            & T-V         &   10.7  &    &     -     &    -  &    -  &  - & 24.2 &  \\ 
				% ImageBind \citep{Girdhar2023ImageBindOE}  & T-V & 36.8    &    & -        & -     & -    & -  & -&-\\
				InternVideo-L \citep{wang2022internvideo}  & T-V & 53.1 &    54.4   & 57.9     & 59.1    & 62.2      & 62.8 & 69.8 & 80.6 \\

				mPLUG-2 \citep{Xu2023mPLUG2AM}  & T-V & 53.1 &    -   & 56.4     & -     & -      & -  & -&-\\
				
				ViCLIP \citep{wang2024internvid}  & T-V & 52.5  & 51.8      & 49.4      & 50.2     & 49.8    & 48.1 & - & - \\
				T-MASS \citep{Wang2024TextIM}  & T-VA & 52.7 &    -   & 53.3     & -     & -      & -  & 65.6 &-\\
				VALOR-L \citep{liu2024valor}  & T-VAS & 54.4 &    -   & 57.6     & -     & 63.4      & -  & 76.9&-\\
				VideoCLIP-XL  \citep{wang2024videoclip}&T-V&54.6&54.0&62.3&62.7&58.4&59.2&-&-\\
				TempMe \citep{shen2025tempme}&T-V&49.0&47.6&48.0&48.4&44.9&45.3&69.6&71.8\\
				\midrule
				\rowcolor{customred} VAST \citep{Chen2023VASTAV}   & T-VA & 55.8   &  57.6  & 65.6 & 62.0 & 68.8 & 66.7  & 86.9 & 84.1  \\
				\rowcolor{customred} VAST \citep{Chen2023VASTAV}   & T-VAS & 56.6   &   57.6  & - & - & - & -  & 87.5 & 84.0 \\

				\midrule
				\rowcolor{customgreen}	BaryBind   & T-V & 57.4 &  57.8  &  67.2   & 64.6  &  67.3   & 65.2 &    85.0  & 82.4 \\
				\rowcolor{customgreen}	BaryBind   & T-VA &  61.0  & 61.4  &  \textbf{69.6}   & \textbf{66.3}  & \textbf{73.5}    & \textbf{70.2}  & 88.2 & 85.2 \\
				\rowcolor{customgreen} 	BaryBind         & T-VAS  &  \textbf{65.7}   & \textbf{65.4}  &   -    & -  &  -    & - & \textbf{89.5}  &  \textbf{86.4}   \\
				\bottomrule
		\end{tabular}}
	\end{table*}
	
	\subsection{Clarification on Distinction from Prior Anchor-based Approaches}
	\label{distinction}
	Our core contribution is the introduction of Wasserstein Barycenter (WB) optimization for establishing a balanced representation space. This differs fundamentally from pure anchor-based alignment approaches (e.g., GRAM and Triangle), by introducing an additional optimization process prior to alignment. This methodology is presented in Section~3, especially Section~3.2 with theoretical analysis. Specifically, BaryBind introduces a learnable WB anchor via WB optimization. The WB optimization transports modality-specific distributions toward a barycenter, defined as the distribution closest to all modality-specific distributions under the Wasserstein distance, ensuring multimodal semantic consensus for subsequent volumetric alignment. Both comparisons and ablations validate that the WB optimization is the main contributor to the superior performance.
	
	On the other hand, along the line of alignment, GRAM and Triangle define a parallelotope and a triangle, respectively, both centered at the origin. In contrast, BaryBind constructs a \textbf{simplex} spanned by the WB anchor and gap vectors. Contrasting this volume enforces paired multimodal data toward a shared representation. This mechanism provides an optimized geometric reference (the WB anchor), followed by contrastive alignment guided by the WB anchor, which is substantially different from pure anchor-based methods.
	
	The advantage of the WB-based representation is further supported by results across multiple tasks, including multimodal retrieval (consistently smaller gaps across benchmarks, and improved robustness under missing-modality settings), multimodal classification, and video-QA. Ablations  further show that removing WB leads to consistent performance drops, indicating that it is the primary contributor.
	
		\subsection{Scalability of BaryBind with Increasing Modality Number}
	\begin{table}[!h]
		\centering
		\caption{Computation time (in seconds) of similarity metrics vs. number of modalities,  measured on an NVIDIA A100 GPU with batch size $B=64$ and embedding dimension $D=512$.}
		\label{computation}
		\resizebox{\textwidth}{!}{
			\begin{tabular}{lcccccc}
				\midrule
				\textbf{Number of modality } $n$ & \textbf{2} & \textbf{3} & \textbf{4} & \textbf{5} & \textbf{10} & \textbf{20} \\
				\midrule
				Pairwise cosine similarity & $3.0 \times 10^{-7}$ & $7.0 \times 10^{-7}$ & $1.0 \times 10^{-6}$ & $1.3 \times 10^{-6}$ & $2.9 \times 10^{-6}$ & $4.9 \times 10^{-6}$ \\
				Barycenter simplex volume    & $4.9 \times 10^{-6}$ & $6.8 \times 10^{-6}$ & $5.9 \times 10^{-6}$ & $9.8 \times 10^{-6}$ & $3.6 \times 10^{-5}$ & $8.8 \times 10^{-5}$ \\
				\bottomrule
		\end{tabular}}
	\end{table}
	\begin{wrapfigure}[16]{R}{0.35\textwidth}
		\centering
		\vspace{-0.4cm}
		\includegraphics[width=0.9\linewidth]{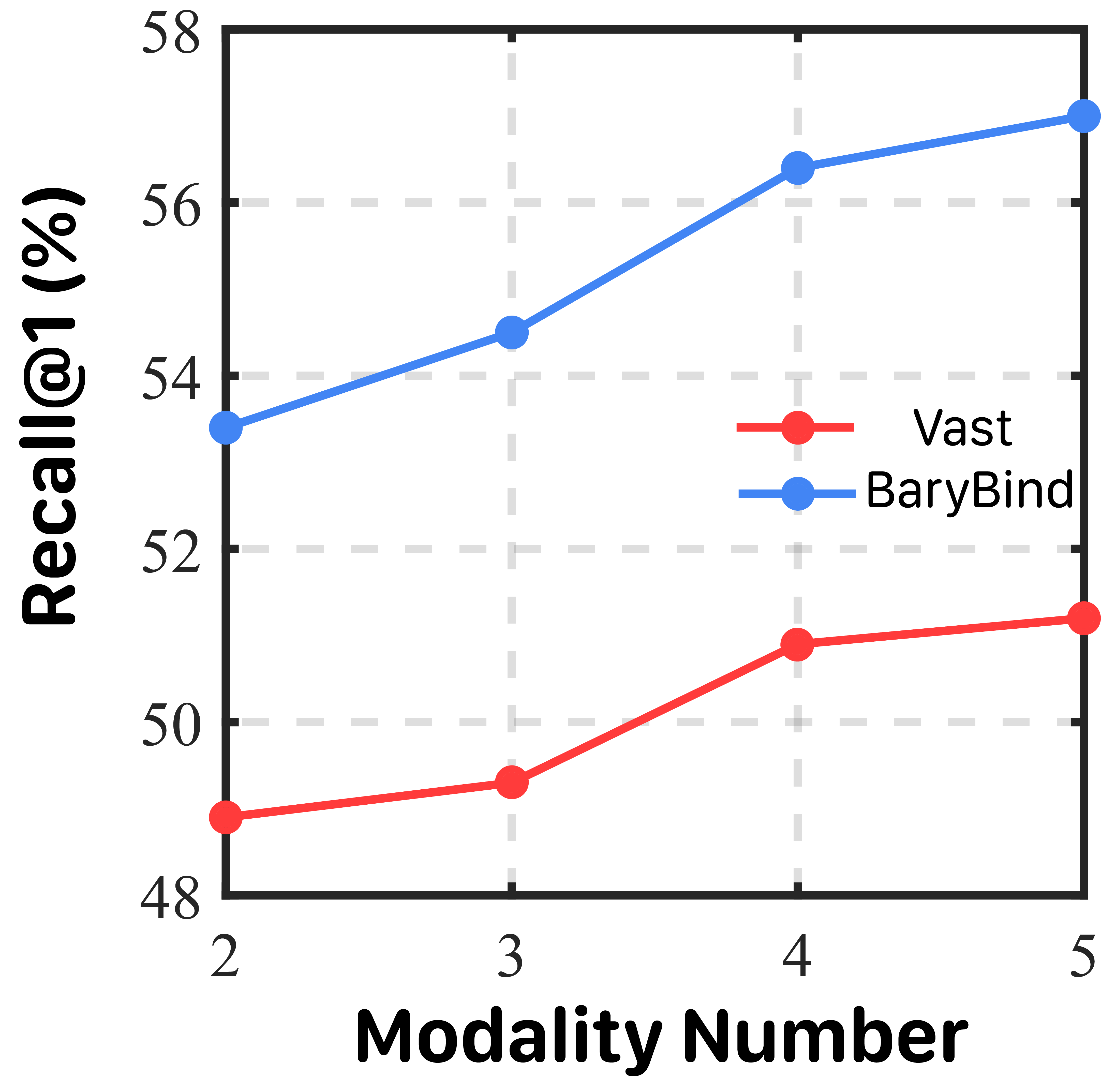}
		\caption{Zero-shot text-to-video retrieval R@1 results as scaling from 2 (T-V) modalities to 5 (T-VASD) modalities.}
		\label{fig:scalability}
	\end{wrapfigure}
	We evaluate the computation efficiency of different similarity metrics with varying modality number $n$. For each metric, we sample $B=64$ sets of $n$ vectors in $\mathbb{R}^D$ with $D=512$, simulating multimodal embeddings. As shown in Table~\ref{computation}, the simplex volume remains computationally efficient and scales reasonably as the modality number increases. The negligible overhead incurred when extending to more modalities highlights the volume as a non-trivial and scalable metric for assessing $n$-modality alignment.
	
	To evaluate how BaryBind scales from bimodal (e.g., text-vision) setups to richer multimodal settings, we progressively expand the input from a basic text-video (T-V) pair to more complex configurations on the MSR-VTT dataset: text-video-audio (T-VA), text-video-audio-subtitle (T-VAS), and finally text-video-audio-subtitle-depth (T-VASD). The depth modality is derived using ChronoDepth~\citep{shao2025learning} and integrated via an additional lightweight head attached to the vision encoder. As shown in Fig.~\ref{fig:scalability}, BaryBind consistently improves Recall@1 as the number of modalities increases, significantly outperforming the VAST baseline on MSR-VTT across all configurations. This highlights the scalability of BaryBind in practical multimodal understanding as it effectively integrates more modalities to shape a richer semantic space. 
	
	\subsection{Robustness on different WB initializers}
	\label{sec:anchor_rationale}

	BaryBind learns the \textbf{Wasserstein Barycenter (WB)} anchor $\boldsymbol {b}$ that serves as a balanced representation aiming to approximate the ideal geometric semantic center.  We utilize the specific modality as a semantic warm-start to initialize the barycenter mapping, and the  $\boldsymbol{b}$ is derived from the initial specific modality $\boldsymbol{m_0}$, i.e.,: $\boldsymbol{b} = T_\theta(\boldsymbol m_{0})$. This allows us to leverage the rich semantic priors to accelerate convergence, while the MWB objective ensures the final representation shifts away from the modality bias towards the geometric center. Crucially, although $T_\theta$ takes $m_0$ as input, its optimization gradients are derived from the Wasserstein distance to all $K$ modalities simultaneously. Therefore, the direction of the update actively pulls the embedding out of the modality-specific manifold.
	
	To verify the robustness of the WB initializer choice, we perform ablation on using vision, text, and  T+V+A (arithmetic fusion) as the initializer across multimodal tasks. From Tables~\ref{tab:anchor_tasks} and \ref{tab:msrvtt_qa}, we observe two findings. (1) Simply choosing a specific modality or fused embedding as the anchor without our barycenter optimization leads to inferior performance and modality imbalance. (2) With WB-driven optimization, all initializations (including the symmetric T+V+A fused initialization) gain clear performance boosts, indicating that the WB formulation itself, rather than the initial anchor choice, is the key contributor to balanced performance. Notably, the text-driven WB performs comparably to the fused-driven and video-driven ones (e.g., 55.6\% vs. 55.8\% and 55.2\% on V+A setting of VGGSound), reflecting the high robustness of BaryBind to different initializations.
	
	\begin{table}[!h]
		\centering
		\caption{
			\textbf{Ablation on the anchor selection under the T-VA setting.} }
		\setlength{\tabcolsep}{8pt}
		\resizebox{1.0\linewidth}{!}{
			\begin{tabular}{clllllll}
				\toprule
				\multirow{3}{*}{Anchor Modality} 
				& \multicolumn{3}{c}{VGGSound} 
				& \multicolumn{4}{c}{MSR-VTT}  \\
				\cmidrule(lr){2-4} \cmidrule(lr){5-8}
				& V 
				& A
				& V + A
				& \multicolumn{2}{c}{T2V} 
				& \multicolumn{2}{c}{V2T} \\
				\cmidrule(lr){2-2} 	\cmidrule(lr){3-3} \cmidrule(lr){4-4}  	\cmidrule(lr){5-6} 	\cmidrule(lr){7-8}
				& Acc@1 & Acc@1 & Acc@1 & R@1 & R@5 & R@1 & R@5 \\
				\midrule
				Video (w/o WB opt.)& 45.8 & 39.6 & 46.5 & 50.9 & 72.6 & 47.4 & 74.8\\
				Text (w/o WB opt.)& 46.3 & 40.3 & 48.1 & 51.5 & 72.2 & 46.8 & 73.1 \\
				T+V+A (w/o WB opt.) & 46.1 & 41.8 & 50.5 & 51.2 & 73.5 & 46.9 & 74.5 \\
				\midrule
				Video-initialized WB 
				& 47.9 \textbf{\textbl{\tiny\!(+2.8)}} 
				& 44.8 \textbf{\textbl{\tiny\!(+5.2)}} 
				& 55.2 \textbf{\textbl{\tiny\!(+8.7)}} 
				& 54.9 \textbf{\textbl{\tiny\!(+4.0)}} 
				& 74.6 \textbf{\textbl{\tiny\!(+2.0)}} 
				& 52.9 \textbf{\textbl{\tiny\!(+5.5)}} 
				& {78.5 \textbl{\tiny\!(+3.7)}} \\
				
				Text-initialized WB  
				& \textbf{48.3 \textgr{\tiny\!(+2.0)}} 
				& \textbf{45.6 \textgr{\tiny\!(+5.3)}} 
				& 55.6 \textbf{\textgr{\tiny\!(+7.5)}} 
				& 56.0 \textbf{\textgr{\tiny\!(+4.5)}} 
				& \textbf{75.4 \textgr{\tiny\!(+3.2)}} 
				& 53.2 \textbf{\textgr{\tiny\!(+6.4)}} 
				& 78.1 \textbf{\textgr{\tiny\!(+5.0)}} \\
				
				T+V+A initialized WB  
				& \textbf{48.3 \textrd{\tiny\!(+2.2)}} 
				& 45.2 \textbf{\textrd{\tiny\!(+3.4)}} 
				& \textbf{55.8 \textrd{\tiny\!(+5.3)}} 
				& \textbf{56.1 \textrd{\tiny\!(+4.9)}} 
				& 75.2 
				& \textbf{53.6 \textrd{\tiny\!(+6.7)}} 
				& \textbf{78.8 \textrd{\tiny\!(+4.3)}}\\
				\bottomrule
		\end{tabular}}
		\vspace{-0.4cm}
		\label{tab:anchor_tasks}
	\end{table}

	\begin{table}[!h]
		\centering
		\caption{Comparison of arithmetic fusion and WB-based models on MSR-VTT and VideoQA under different modalities}
		\label{tab:msrvtt_qa}
		\footnotesize
		\setlength{\tabcolsep}{5pt}
		\begin{tabular}{lccccc}
			\toprule
			\multirow{2}{*}{Anchor} & \multirow{2}{*}{Modality} & \multicolumn{2}{c}{Retrieval (R@1)} & \multicolumn{2}{c}{VideoQA (Acc)} \\
			\cmidrule(lr){3-4} \cmidrule(lr){5-6}
			& & T2V & V2T & MSRVTT-QA & MSVD-QA \\
			\midrule
			T+V+A (w/o WB opt.) & T-VA & 51.2 & 46.9 & 47.6 & 55.1 \\
			T+V+A-initialized WB & T-VA & \textbf{56.1} & \textbf{53.6} & \textbf{52.0} & \textbf{60.3} \\
			Text-initialized WB & T-VA & \textbf{56.0} & \textbf{53.2} & \textbf{52.2} & \textbf{60.3} \\
			\midrule
			T+V+A+S (w/o WB opt.) & T-VAS & 51.7 & 47.3 & 50.3 & 57.9 \\
			T+V+A+S-initialized WB & T-VAS & \textbf{57.1} & \textbf{53.8} & \textbf{55.7} & \textbf{63.8} \\
			Text-initialized WB & T-VAS & \textbf{57.2} & \textbf{53.9} & \textbf{55.4} & \textbf{63.6} \\
			\midrule
			T+V+A+S+D (w/o WB opt.) & T-VASD & 51.8 & 47.2 & $-$ & $-$ \\
			T+V+A+S+D-initialized WB & T-VASD & \textbf{57.9} & \textbf{55.0} & $-$ & $-$ \\
			Text-initialized WB & T-VASD & \textbf{57.8} & \textbf{55.1} & $-$ & $-$ \\
			\bottomrule
		\end{tabular}
	\end{table}
	
		\subsection{Normalization Strategies for Polytope Volume Metric}
	
	\begin{wraptable}[13]{r}{0.45\textwidth} % 表格右侧，宽度为页面宽度的45%
		\centering
		\vspace{-0.1cm}
		\caption{Zero-shot generalization on MSR-VTT under different normalization strategies for the simplex volume.}
		\resizebox{0.85\linewidth}{!}{
			\begin{tabular}{lccc}
				\toprule
				Normalization & Setting & T2V & V2T \\
				\midrule
				$(V/2!)^{1/2}$ & T-V   & 53.9 & 52.2 \\
				$V$       & T-V   & 54.1 & 52.0 \\
				$(V/3!)^{1/3}$ & T-VA  & 54.8 & 52.9 \\
				$V$       & T-VA  & 55.1 & 52.7 \\
				$(V/4!)^{1/4}$ & T-VAS & 57.3 & 54.2 \\
				$V$       & T-VAS & 57.2 & 54.5 \\
				\bottomrule
		\end{tabular}}
		\label{tab:volume_normalization}
	\end{wraptable}
	
	While the raw barycenter simplex volume $V$ serves as a global alignment metric, it depends on the number of modalities and embedding dimensionality, which limits interpretability and cross-setting comparability. To address this, we consider normalization strategies such as $(V/n!)^{1/n}$ (viewing the WB simplex as an $n$-simplex), which scale the volume to better reflect per-modality contributions and make the metric more comparable across different modality configurations.
	
		Table~\ref{tab:volume_normalization} reports zero-shot retrieval results on MSR-VTT under different normalization strategies. We observe that applying normalization slightly reduces the absolute retrieval scores compared to the raw volume $V$, likely because normalization reduces the magnitude of the gradient signal from the volume metric, slightly weakening the alignment supervision. Nonetheless, the overall performance remains stable, indicating that $(V/n!)^{1/n}$ provides a reliable and interpretable metric without significantly sacrificing retrieval accuracy.
		
	\subsection{Hyperparameter Sensitivity Analysis}
	\label{sec:sens}
	
	We study the sensitivity of the loss weights in the total objective on the MSR-VTT validation set:
	\begin{equation}
		\mathcal{L} = \mathcal{L}_{\text{MWB}} + \alpha_1 \mathcal{L}_{\text{BVC}} + \alpha_2 \mathcal{L}_{\text{DAM}}.
	\end{equation}
	
	As reported in Table~\ref{tab:alpha1_sensitivity}, adjusting $\alpha_1$ on MSR-VTT shows that increasing the weight of the BVC loss gradually improves retrieval performance, and the best results are achieved when $\alpha_1=1$. A further increase leads to a performance decline, indicating that over-restricting barycentric geometry may weaken instance-level discrimination.
	
	\begin{table*}[!h]
		\centering
		\begin{minipage}{0.48\linewidth}
			\centering
			\caption{Sensitivity analysis of $\alpha_1$ on MSR-VTT validation set.}
			\label{tab:alpha1_sensitivity}
			\resizebox{0.9\linewidth}{!}{
				\begin{tabular}{c|ccccc}
					\toprule
					$\alpha_1$ & 0.3 & 0.6 & 1 & 1.5 & 1.8 \\
					\midrule
					T2V (R@1) & 56.2 & 55.7 & \textbf{56.5} & 56.0&55.3  \\
					V2T (R@1) & 58.0 & 57.5 & \textbf{58.3} & 57.9&56.5 \\
					\bottomrule
			\end{tabular}}
		\end{minipage}
		\hfill
		\begin{minipage}{0.48\linewidth}
			\centering
			\caption{Sensitivity analysis of $\alpha_2$ on MSR-VTT validation set.}
			\label{tab:alpha2_sensitivity}
			\resizebox{0.9\linewidth}{!}{
				\begin{tabular}{c|ccccc}
					\toprule
					$\alpha_2$ & 0.02 & 0.06 & 0.1 & 0.14&0.18 \\
					\midrule
					T2V (R@1)  & 55.6 & 55.8 & \textbf{56.5} & 56.2&56.0  \\
					V2T (R@1) & 57.6 & 57.2 & \textbf{58.3} & 57.9&57.7  \\
					\bottomrule
			\end{tabular}}
		\end{minipage}
	\end{table*}

	Table~\ref{tab:alpha2_sensitivity} reveals that $\alpha_2$ also has an optimal range, with $\alpha_2=0.1$ yielding the strongest retrieval performance. Excessively small or large weights cause suboptimal alignment due to either under-constrained or modality-biased instance matching.
	
	Overall, $\alpha_1=1$ and $\alpha_2=0.1$ deliver the highest multimodal alignment on the MSR-VTT validation set, demonstrating the complementary strengths of BVC and DAM.

	\subsection{Training from Scratch Evaluation}
	\label{scratch}
	\begin{wraptable}[11]{r}{0.45\textwidth} % 表格右侧，宽度为页面宽度的45%
		\centering
		\vspace{-0.5cm}
		\caption{Training-from-scratch  T2V and V2T results on MSR-VTT and ActivityNet. The training follows the T-VA setting.}
				\vspace{-0.1cm}
		\label{tab:scratch}
		\resizebox{\linewidth}{!}{
			\begin{tabular}{lcccc}
				\toprule
				\multirow{2}{*}{Methods}&\multicolumn{2}{c}{MSR-VTT}  & \multicolumn{2}{c}{ActivityNet}  \\
				\cmidrule(lr){2-3} \cmidrule(lr){4-5}
				&T2V&V2T&T2V&V2T\\
				\midrule
				VAST \citep{Chen2023VASTAV}     & 35.0 & 39.1 & 26.7 & 25.2 \\
				LanguageBind \cite{Zhu2023LanguageBindEV} &31.8&36.4&29.3&26.1\\
				GRAM \cite{cicchetti2025gramian}    & 35.7 & 39.4 & 27.2 &  29.1  \\
				Triangle \cite{cicchetti2025triangle} & 39.4 & 40.1 & 27.8 & 29.3 \\
				\midrule
				BaryBind & \textbf{44.8} & \textbf{43.2} & \textbf{33.6} & \textbf{34.2} \\
				\bottomrule
		\end{tabular}}
	\end{wraptable}
	
	Due to the unavailability of full VAST27M dataset \citep{Chen2023VASTAV}, we adopt continuing training based on the VAST model in \S \ref{4.1}.  Here we further evaluate the models trained from scratch on MSR-VTT and ActivityNet. Four representative models: VAST~\cite{Chen2023VASTAV}, which performs pairwise alignment with cosine-similarity heads and lightweight fusion layers; GRAM~\cite{cicchetti2025gramian}, Triangle~\cite{cicchetti2025triangle}, and the proposed BaryBind.
	Table~\ref{tab:scratch} reports the R@1 results for both text-to-video (T2V) and video-to-text (V2T) retrieval. BaryBind achieves the highest performance across all settings, surpassing VAST, GRAM, and Triangle on both datasets. On MSR-VTT, BaryBind improves over VAST by +6.4\% (T2V) and +3.8\% (V2T), while also outperforming GRAM and Triangle by comparable margins. A similar pattern appears on ActivityNet, where BaryBind reaches 30.8\% (T2V) and 32.4\% (V2T). The training curves in Fig. ~\ref{fig:loss_curve} show that BaryBind continues to refine the latent space throughout training, whereas the other methods plateau earlier. These observations indicate that a barycenter-based objective provides stronger structural guidance when no pretrained semantic prior is available, enabling the model to form a more coherent multimodal representation.

	\subsubsection{Effect of Relaxing the Congruence Constraint}
	
	The proposed formulation enforces the congruence constraint
	$\sum_k \lambda_k f_k=0$
	through a hard parameterization. To investigate the necessity of strictly enforcing the congruence constraint, we replace the hard parameterization with a soft penalty formulation and optimize the following objective:
	\begin{equation}
		\mathcal{L}_{\mathrm{soft}}
		=
		\mathcal{L}_{\mathrm{MWB}}
		+
		\eta
		\left\|
		\sum_k\lambda_k f_k
		\right\|_2^2 .
	\end{equation}
	During alternating optimization, we record the objective function for updating the potentials:
	\begin{equation}
		\mathcal{L}_{f}
		=
		-\sum_k \lambda_k
		\mathbb{E}_{m_k\sim\mathbb{P}_k}
		\left[
		f_k(T_\theta(m_k))
		\right]
		+
		\eta
		\left\|
		\sum_k\lambda_k f_k
		\right\|_2^2 ,
	\end{equation}
	
		\begin{wraptable}[9]{r}{0.48\textwidth}
		\vspace{-0.15in}
		\centering
		\caption{Stability analysis of potential optimization during alternating updates.}
		\label{tab:potential_stability}
		\small
		\begin{tabular}{c|c|c|c}
			\toprule
			Iterations & Mean $\mathcal{L}_{f}$ & Range & $\Delta$\\
			\midrule
			0--500 & -0.376 & [-1.000,-0.154] & 0.846\\
			500--1000 & -0.237 & [-0.309,-0.171] & 0.138\\
			1000--2000 & -0.181 & [-0.242,-0.134] & 0.108\\
			2000--3000 & -0.098 & [-0.138,-0.045] & 0.093\\
			3000--10000 & -0.030 & [-0.049,-0.011] & 0.038\\
			\bottomrule
		\end{tabular}
		\vspace{-0.1in}
	\end{wraptable}
	
	where $\eta=10$ is used in the experiment. The optimization dynamics are summarized in Table~\ref{tab:potential_stability}. Although the early optimization stage exhibits noticeable fluctuations, the potential optimization gradually stabilizes after approximately 3,000 iterations without persistent divergence.

	We further compare the hard and soft formulations in Table~\ref{tab:soft_constraint}. The soft penalty achieves comparable performance with only marginal degradation, demonstrating that the congruence constraint mainly enables a stable optimization.

		\begin{table}[!h]
		\centering
		\caption{Comparison between hard and soft congruence constraints.}
		\label{tab:soft_constraint}
		\small
		\begin{tabular}{c|c|c}
			\toprule
			Constraint & T2V/V2T R@1 & A+C Acc@1\\
			\midrule
			Hard constraint & 56.1/53.6 & 55.6\\
			Soft constraint & 55.9/53.2 & 55.1\\
			\bottomrule
		\end{tabular}
		\vspace{-0.1in}
	\end{table}
	
	\subsubsection{Stability of WB Map Adaptation under Encoder Distribution Shifts}
	
	During encoder fine-tuning, the feature distributions continuously evolve. We therefore analyze the adaptation dynamics of the WB map $T_\theta$. Specifically, we track the WB mapping loss:
	
	\begin{equation}
		\mathcal{L}_{T}
		\triangleq
		\sum_{k=1}^{K}
		\lambda_k
		\mathbb{E}_{m_k\sim\mathbb{P}_k}
		\left[
		\left\|
		m_k-T_\theta(m_0)
		\right\|_2^2
		-
		f_{\omega_k}(T_\theta(m_0))
		\right].
	\end{equation}
	
	\begin{wraptable}{r}{0.48\textwidth}
		\vspace{-0.15in}
		\centering
		\caption{Loss analysis of WB map $T_\theta$ adaptation during encoder fine-tuning.}
		\label{tab:wb_adaptation}
		\small
		\begin{tabular}{c|c|c|c}
			\toprule
			Iterations & Mean $\mathcal{L}_{T}$ & Range & $\Delta$\\
			\midrule
			0--500 &0.971 &[0.913,1.000]&0.087\\
			500--1000&0.602&[0.317,0.913]&0.597\\
			1000--2000&0.136&[0.073,0.330]&0.257\\
			2000--3000&0.052&[0.020,0.095]&0.075\\
			3000--10000&0.029&[0.000,0.022]&0.022\\
			\bottomrule
		\end{tabular}
		\vspace{-0.1in}
	\end{wraptable}
	
	As shown in Table~\ref{tab:wb_adaptation}, the WB map gradually adapts to the evolving encoder distributions. During the initial stage, the mapping loss changes as the WB map responds to encoder updates. A rapid adaptation phase is observed between 500 and 2,000 iterations, where the loss decreases substantially. After approximately 3,000 iterations, the variation becomes significantly smaller, indicating that the WB map reaches a stable adaptation regime with the evolving feature distributions.
		
	\subsection{Computational Overhead Comparison During Training}
	\label{overhead}
	We provide a comparison of computational overhead among VAST, GRAM, and BaryBind. All experiments are conducted on 2$\times$NVIDIA A100 80GB GPUs with mixed-precision (FP16/AMP). The additional computation in BaryBind mainly comes from the Wasserstein barycenter optimization and auxiliary alignment losses. The increase in per-step time remains moderate while providing improved multimodal alignment performance.
	
	\begin{table}[h]
		\centering
		\caption{Computational overhead comparison during training.}
		\label{tab:efficiency}
		\begin{tabular}{lccccc}
			\toprule
			\textbf{Model} & \textbf{Params} & \textbf{Batch size} & \textbf{Forward+Backward} & \textbf{Steps/Epoch} & \textbf{Time/Epoch} \\
			\midrule
			VAST & 1.28B & 64 & $\sim$8.8s & 2344 & $\sim$5.7h \\
			GRAM & 1.30B & 64 & $\sim$9.4s & 2344 & $\sim$6.1h \\
			BaryBind & 1.34B & 64 & $\sim$9.7s & 2344 & $\sim$6.3h \\
			\bottomrule
		\end{tabular}
	\end{table}
	
\subsection{Semantic Consensus in the Learned WB Space}

\textbf{Empirical illustration.}
To illustrate the semantic consensus property of the learned WB space, we
perform cross-modal retrieval through the WB space. Given a query from one
modality, it is first mapped into the WB space and then used to retrieve
samples from other modalities. As shown in Table~\ref{tab:wb_text_retrieval}
and Table~\ref{tab:wb_video_retrieval}, the retrieved samples consistently
share the same high-level semantic concepts (e.g., cooking and basketball),
while preserving modality-specific variations. These results empirically
demonstrate that the learned WB space captures shared semantic structures
across modalities.

\begin{table}[t]
	\centering
	\caption{Empirical illustration of semantic consensus through
		Text$\rightarrow$WB$\rightarrow$X retrieval.}
	\label{tab:wb_text_retrieval}

	\begin{adjustbox}{width=0.95\textwidth,center}
	\begin{tabular}{llll}
		\toprule
		Query & Retrieval Path & Modality (Retrievaled sample)&Caption \\
		\midrule
		
		\multirow{3}{*}{A man is cooking in the kitchen}
		& \multirow{3}{*}{Text$\rightarrow$WB$\rightarrow$Video}
		& Video8834 &
		A man is cooking in the kitchen; he states he will return later to add to the dish \\
		&& Video9687 &
		A man chopping lobster and taking off the shell \\
		&& Video9038 &
		A person is preparing some food \\
		
		\midrule
		
		\multirow{3}{*}{A man is cooking in the kitchen}
		& \multirow{3}{*}{Text$\rightarrow$WB$\rightarrow$Audio}
		& Audio (Video7376) &
		Broth is being added to a soup pot and stirred with a rubber spatula \\
		&& Audio (Video8834) &
		A man is cooking in the kitchen; he states he will return later to add to the dish \\
		&& Audio (Video7799) &
		A guy chops up garlic and pours it over chicken frying in a pan \\
		
		\midrule
		
		\multirow{3}{*}{A man is cooking in the kitchen}
		& \multirow{3}{*}{Text$\rightarrow$WB$\rightarrow$Subtitle}
		& Subtitle (Video7362) &
		A person is cooking on stage \\
		&& Subtitle (Video9687) &
		A man chopping lobster and taking off the shell \\
		&& Subtitle (Video8772) &
		The man is making a sauce in the kitchen \\
		
		\bottomrule
	\end{tabular}\end{adjustbox}
\end{table}

\begin{table}[t]
	\centering
	\caption{Empirical illustration of semantic consensus through
		Video$\rightarrow$WB$\rightarrow$X retrieval.}
	\label{tab:wb_video_retrieval}

			\begin{adjustbox}{width=0.95\textwidth,center}
	\begin{tabular}{llll}
		\toprule
		Query & Retrieval Path & Modality (Retrieved sample) & Caption \\
		\midrule
		
		\multirow{4}{*}{video of a basketball event}
		& \multirow{4}{*}{Video$\rightarrow$WB$\rightarrow$Text}
		& Text (Video7763) &
		People are playing basketball \\
		&& Text (Video9029) &
		A player is putting a basketball into the basket from distance \\
		&& Text (Video8124) &
		Basketball players making a shot in the last seven seconds \\
		&& Text (Video9013) &
		Basketball highlights of players scoring \\
		
		\midrule
		
		\multirow{3}{*}{video of a basketball event}
		& \multirow{3}{*}{Video$\rightarrow$WB$\rightarrow$Audio}
		& Audio (Video8124) &
		Basketball players making a shot in the last seven seconds \\
		&& Audio (Video7763) &
		People are playing basketball \\
		&& Audio (Video9013) &
		Basketball highlights of players scoring \\
		
		\midrule
		
		\multirow{3}{*}{video of a basketball event}
		& \multirow{3}{*}{Video$\rightarrow$WB$\rightarrow$Subtitle}
		& Subtitle (Video8124) &
		Basketball players making a shot in the last seven seconds \\
		&& Subtitle (Video9013) &
		Basketball highlights of players scoring \\
		&& Subtitle (Video8652) &
		A man runs into the crowd when trying to catch a basketball \\
		
		\bottomrule
	\end{tabular}\end{adjustbox}
\end{table}
	
	\subsection{More Quantitative Results on Scaling to More Modalities}
	
	To systematically evaluate the scalability of multimodal models with respect to the number of input modalities, we conduct experiments under four configurations, progressively increasing from two modalities (text and video) to five (text, video, audio, subtitle, and depth). This stepwise setup enables a controlled analysis of how each additional modality affects performance and alignment quality. For consistency and interpretability, text is used as the anchor modality throughout.

	As shown in Table~\ref{tab:scaling_modalities}, both VAST and BaryBind benefit from the inclusion of additional modalities, demonstrating improved performance on MSR-VTT and VATEX in terms of both T2V and V2T retrieval. Notably, BaryBind consistently outperforms VAST across all modality configurations and datasets, highlighting its stronger scalability and generalization capacity.

	\begin{table}[!h]
		\centering
		\caption{
			Downstream performance with increasing number of modalities.} 
		\setlength{\tabcolsep}{10pt}
		\renewcommand{\arraystretch}{1.0}  % 设置行距
		\resizebox{0.95\linewidth}{!}{
			\begin{tabular}{ccccccccccccc}
				\toprule
				\multicolumn{5}{c}{\multirow{2}{*}{Modalities}}&\multicolumn{4}{c}{MSR-VTT}&
				\multicolumn{4}{c}{VATEX}\\
				\cmidrule(lr){6-9} \cmidrule(lr){10-13}
				&&&&&\multicolumn{2}{c}{VAST}&\multicolumn{2}{c}{BaryBind}&\multicolumn{2}{c}{VAST}&\multicolumn{2}{c}{BaryBind}\\
				\midrule
				Text&Video&Audio&Sub.&Depth&T2V&V2T&T2V&V2T&T2V&V2T&T2V&V2T\\\midrule
				\cmark&\cmark&\xmark&\xmark&\xmark&48.7&43.2&54.2&52.0&78.8&77.0&82.3&79.8\\
				\cmark&\cmark&\cmark&\xmark&\xmark&49.3&43.7&56.0&53.2&80.0&77.3&84.2&81.3\\
				\cmark&\cmark&\cmark&\cmark&\xmark&50.9&47.9&57.2&53.6&82.1&78.7&84.6&83.5\\
				\cmark&\cmark&\cmark&\cmark&\cmark&51.2&49.3&57.9&54.4&82.4&79.2&84.9&83.8\\
				\bottomrule
		\end{tabular}}
		
		\vspace{-0.2cm}
		\label{tab:scaling_modalities}
	\end{table}
	
	\subsection{Additional Evaluations on VideoQA and Cross-modal Generation}
	\label{appendix:qa_gen}
	
	To further demonstrate the versatility and broader impact of BaryBind, we conduct additional experiments on Video Question Answering (VideoQA) and cross-modal generation tasks.
	
	\textbf{VideoQA.} Following the data and captioner setting of VAST \citep{Chen2023VASTAV}, we evaluate BaryBind on the VideoQA benchmarks MSRVTT-QA and MSVD-QA. As shown in Table~\ref{tab:videoqa}, BaryBind achieves competitive accuracy on both benchmarks and outperforms recent global alignment methods. Notably, the superior performance on QA tasks highlights the effectiveness of the learned representations in capturing fine-grained semantics.
	
	\textbf{Cross-modal generation.} We use the VGGSound dataset to evaluate the generation performance, which consists of approximately 200K video clips annotated with 309 sound classes. We utilize the captioner provided by VAST to generate video captions, using 1024 videos for testing and the remainder for training. We adopt a single encoder–decoder configuration compatible with UnCLIP-style generators, specifically a CLIP ViT-H/14 image–text encoder paired with the ImageBind audio encoder, compatible with the Stable UnCLIP decoder. We evaluate text-to-image (T2I), audio-to-image (A2I), and text-audio-to-image (TA2I) generation performance.
	
	\begin{table}[h]
		\centering
		\footnotesize
		\begin{minipage}[t]{0.48\textwidth}
			\centering
			\caption{Evaluation on VideoQA benchmarks. Methods utilizing audio or subtitle modalities are \textbf{bolded}.}
			\label{tab:videoqa}
			\vspace{5pt} 
			\begin{tabular}{lcc}
				\toprule
				Method & MSRVTT-QA & MSVD-QA \\
				\midrule
				mPLUG-2 & 48.0 & 58.1 \\
				VALOR-L & 49.2 & 60.0 \\
				\textbf{VAST} & 50.1 & 60.2 \\
				\textbf{GRAM} & 52.3 & 61.5 \\
				\textbf{Triangle} & 52.0 & 61.9 \\
				\rowcolor{customgreen} \textbf{BaryBind} & \textbf{55.4} & \textbf{63.6} \\
				\bottomrule
			\end{tabular}
		\end{minipage}
		\hfill
		\begin{minipage}[t]{0.48\textwidth}
			\centering
			\caption{Cross-modal generation results on VGGSound (FID $\downarrow$). T2I, A2I, and TA2I denote text-to-image, audio-to-image, and joint text-audio-to-image generation, respectively.}
			\label{tab:generation}
			\vspace{5pt}
			\begin{tabular}{lccc}
				\toprule
				Method & T2I & A2I & TA2I \\
				\midrule
				ImageBind & 50.1 & 53.6 & 46.8 \\
				VAST & 48.1 & 57.2 & 50.6 \\
				GRAM & 46.7 & 56.5 & 45.2 \\
				Triangle & 46.4 & 54.5 & 44.3 \\
				\rowcolor{customgreen} BaryBind & \textbf{43.6} & \textbf{50.3} & \textbf{38.6} \\
				\bottomrule
			\end{tabular}
		\end{minipage}
	\end{table}

	Qualitative examples of the generation results are provided in the right part of Figure~\ref{teaser}, illustrating the semantic consistency achieved by BaryBind.

	\subsection{Visualization of Top-1 Retrieval Results}
	
	To qualitatively assess the retrieval performance of BaryBind, we visualize top-1 text-to-video results compared with two strong baselines in Fig.~\ref{fig:retvis}. BaryBind integrates text, audio, video, and subtitle modalities during both training and inference. Each row shows five frames from the top-retrieved video along with the query subtitle. In the first example, BaryBind retrieves a beach party scene aligned with the query’s semantic and acoustic mood, while baselines return less relevant results. In the second case, BaryBind accurately matches a gameplay scene described by both visual and subtitle cues, whereas baselines retrieve unrelated content. The third example features a simple emotional phrase, “I’m scared.”, where BaryBind selects an animated video of a fearful cat-dog chase, while others fail to reflect the emotional context or core entities. These results demonstrate BaryBind’s ability to leverage complementary multimodal cues for precise and context-aware retrieval, effectively grounding both semantics and affect across diverse scenarios.

	\subsection{Limitations and Future Work}\label{limitations}
	\textbf{Limitation}: We note that the introduction of the iterative barycenter optimization involves a marginal increase in training computational overhead compared to simple pairwise alignment. However, as shown in Table \ref{tab:efficiency}, this cost is negligible and is well-justified by the notable gains in cross-modal retrieval accuracy.
	
	We consider the integration of BaryBind into large multimodal language models (LLMs) as a direction for future work. Specifically, we are interested in exploring how BaryBind can enhance the multimodal understanding capabilities of these models by providing a unified, balanced token space that is not centered around language. This approach could help alleviate the issue of ``hallucinations" in large multimodal models. Furthermore, by focusing on a more balanced representation across modalities, we aim to promote a deeper understanding of the physical world, moving beyond linguistic reasoning towards true multimodal perception that can more accurately capture and reason about real-world phenomena. In addition, we plan to incorporate reconstruction into the BaryBind framework to enable multimodal understanding and generation, broadening its applicability in both discriminative and generative multimodal tasks.
	
	\subsection{Impact Statement}
	\label{impact}
	This paper advances multimodal learning by leveraging optimal-transport–based barycenter theory to achieve more balanced multimodal representation learning. The potential societal impact includes enabling more efficient and robust approaches for general multimodal understanding. To the best of our knowledge, this work does not pose notable ethical or societal risks. Our goal is to support the responsible development of scalable and balanced multimodal learning methods and theories.
	
	\begin{figure}[!h]
		\centering
		\includegraphics[width=0.85\linewidth]{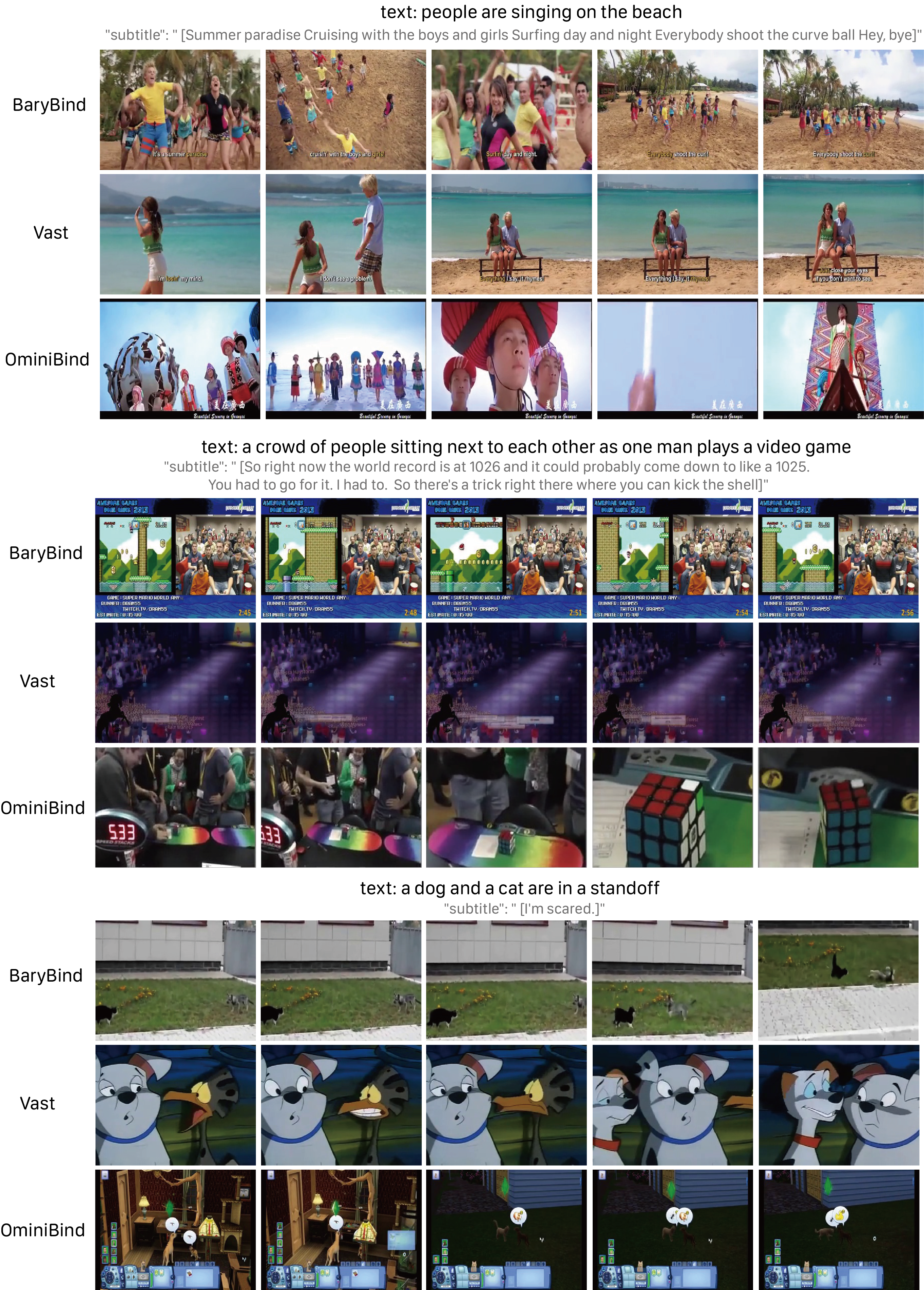}
		\caption{Visual results of text-to-video retrieval. We display 5 frames from the top-1 video.}
		\label{fig:retvis}
	\end{figure}
	
	%%%%%%%%%%%%%%%%%%%%%%%%%%%%%%%%%%%%%%%%%%%%%%%%%%%%%%%%%%%%
	\newpage
	\quad\\
	\newpage
	\input{checklist.tex}

\end{document}

%% file: math_commands.tex
\usepackage{amsmath,amsfonts,bm}

\def\eqref#1{equation~\ref{#1}}
\def\1{\bm{1}}

\DeclareMathAlphabet{\mathsfit}{\encodingdefault}{\sfdefault}{m}{sl}
\SetMathAlphabet{\mathsfit}{bold}{\encodingdefault}{\sfdefault}{bx}{n}

%% file: checklist.tex
\section*{NeurIPS Paper Checklist}

\begin{enumerate}
	
	\item {\bf Claims}
	\item[] Question: Do the main claims made in the abstract and introduction accurately reflect the paper's contributions and scope?
	\item[] Answer: \answerYes{} % Replace by \answerYes{}, \answerNo{}, or \answerNA{}.
	\item[] Justification: The main contributions are summarized in the introduction Section. These contributions are consistent with the results we presented in Section 3 and Section 4.
	
	\item {\bf Limitations}
	\item[] Question: Does the paper discuss the limitations of the work performed by the authors?
	\item[] Answer: \answerYes{} % Replace by \answerYes{}, \answerNo{}, or \answerNA{}.
	\item[] Justification: We discuss the limitations and future work in Appendix \ref{limitations}.
	
	\item {\bf Theory assumptions and proofs}
	\item[] Question: For each theoretical result, does the paper provide the full set of assumptions and a complete (and correct) proof?
	\item[] Answer: \answerYes{} % Replace by \answerYes{}, \answerNo{}, or \answerNA{}.
	\item[] Justification: All assumptions are stated before the theorems. Due to space constraints, we provide the complete proofs in Appendix \ref{proof}.
	
	\item {\bf Experimental result reproducibility}
	\item[] Question: Does the paper fully disclose all the information needed to reproduce the main experimental results of the paper to the extent that it affects the main claims and/or conclusions of the paper (regardless of whether the code and data are provided or not)?
	\item[] Answer: \answerYes{} % Replace by \answerYes{}, \answerNo{}, or \answerNA{}.
	\item[] Justification: We provide all necessary details for reproduction in Section 4 and Appendix B.2.

	\item {\bf Open access to data and code}
	\item[] Question: Does the paper provide open access to the data and code, with sufficient instructions to faithfully reproduce the main experimental results, as described in supplemental material?
	\item[] Answer: \answerNo{} % Replace by \answerYes{}, \answerNo{}, or \answerNA{}.
	\item[] Justification: While the data are publicly available, the code is not included in the supplemental material at this stage. However, we provide comprehensive implementation details in Section 4 and Appendix B.2 and algorithm in Appendix B.1 to ensure reproducibility.

	\item {\bf Experimental setting/details}
	\item[] Question: Does the paper specify all the training and test details (e.g., data splits, hyperparameters, how they were chosen, type of optimizer) necessary to understand the results?
	\item[] Answer: \answerYes{} % Replace by \answerYes{}, \answerNo{}, or \answerNA{}.
	\item[] Justification: We specify all training and test details in Section 4 and Appendix B.2.
	
	\item {\bf Experiment statistical significance}
	\item[] Question: Does the paper report error bars suitably and correctly defined or other appropriate information about the statistical significance of the experiments?
	\item[] Answer: \answerNo{} % Replace by \answerYes{}, \answerNo{}, or \answerNA{}.
	\item[] Justification: We use fixed random seeds and deterministic configurations for all experiments to ensure reproducibility. Therefore, error bars are not reported.
	
	\item {\bf Experiments compute resources}
	\item[] Question: For each experiment, does the paper provide sufficient information on the computer resources (type of compute workers, memory, time of execution) needed to reproduce the experiments?
	\item[] Answer: \answerYes{} % Replace by \answerYes{}, \answerNo{}, or \answerNA{}.
	\item[] Justification: We specify the hardware (2$\times$A100 GPUs) and the training duration in Appendix \ref{details}. Further details on computational overhead are provided in Appendix \ref{overhead}.
	
	\item {\bf Code of ethics}
	\item[] Question: Does the research conducted in the paper conform, in every respect, with the NeurIPS Code of Ethics \url{https://neurips.cc/public/EthicsGuidelines}?
	\item[] Answer: \answerYes{} % Replace by \answerYes{}, \answerNo{}, or \answerNA{}.
	\item[] Justification: We have reviewed the NeurIPS Code of Ethics and confirm that our research, including data usage and methodology, adheres to all its guidelines.

	\item {\bf Broader impacts}
	\item[] Question: Does the paper discuss both potential positive societal impacts and negative societal impacts of the work performed?
	\item[] Answer: \answerYes{} % Replace by \answerYes{}, \answerNo{}, or \answerNA{}.
	\item[] Justification: We discuss the potential positive and negative societal impacts of our work in Appendix \ref{impact}.
	
	\item {\bf Safeguards}
	\item[] Question: Does the paper describe safeguards that have been put in place for responsible release of data or models that have a high risk for misuse (e.g., pre-trained language models, image generators, or scraped datasets)?
	\item[] Answer: \answerNA{} % Replace by \answerYes{}, \answerNo{}, or \answerNA{}.
	\item[] Justification: We use only existing, publicly available datasets.
	
	\item {\bf Licenses for existing assets}
	\item[] Question: Are the creators or original owners of assets (e.g., code, data, models), used in the paper, properly credited and are the license and terms of use explicitly mentioned and properly respected?
	\item[] Answer: \answerYes{} % Replace by \answerYes{}, \answerNo{}, or \answerNA{}.
	\item[] Justification: All existing assets (datasets and backbone models) are properly cited in Section 4. And we use them only for academic research in compliance with their original licenses and terms of use.
	
	\item {\bf New assets}
	\item[] Question: Are new assets introduced in the paper well documented and is the documentation provided alongside the assets?
	\item[] Answer: \answerNA{} % Replace by \answerYes{}, \answerNo{}, or \answerNA{}.
	\item[] Justification: This paper does not release new datasets or code assets as part of the current submission.
	
	\item {\bf Crowdsourcing and research with human subjects}
	\item[] Question: For crowdsourcing experiments and research with human subjects, does the paper include the full text of instructions given to participants and screenshots, if applicable, as well as details about compensation (if any)? 
	\item[] Answer: \answerNA{} % Replace by \answerYes{}, \answerNo{}, or \answerNA{}.
	\item[] Justification: This work does not involve crowdsourcing nor research with human subjects.
	
	\item {\bf Institutional review board (IRB) approvals or equivalent for research with human subjects}
	\item[] Question: Does the paper describe potential risks incurred by study participants, whether such risks were disclosed to the subjects, and whether Institutional Review Board (IRB) approvals (or an equivalent approval/review based on the requirements of your country or institution) were obtained?
	\item[] Answer: \answerNA{} % Replace by \answerYes{}, \answerNo{}, or \answerNA{}.
	\item[] Justification: This research does not involve human subjects.
	
	\item {\bf Declaration of LLM usage}
	\item[] Question: Does the paper describe the usage of LLMs if it is an important, original, or non-standard component of the core methods in this research? Note that if the LLM is used only for writing, editing, or formatting purposes and does \emph{not} impact the core methodology, scientific rigor, or originality of the research, declaration is not required.
	%this research? 
	\item[] Answer: \answerNA{} % Replace by \answerYes{}, \answerNo{}, or \answerNA{}.
	\item[] Justification: The core method in this research does not invovle LLMs as an important, original, or non-standard component.
	
\end{enumerate}